%% file: bare_adv.tex
\documentclass[10pt,journal,compsoc]{IEEEtran}

\ifCLASSOPTIONcompsoc
\usepackage[utf8]{inputenc} 
\usepackage[T1]{fontenc}    
\usepackage{hyperref}       
\usepackage{url}            
\usepackage{booktabs}       
\usepackage{amsfonts}       
\usepackage{nicefrac}       
\usepackage{microtype}      
\usepackage{bm}
\usepackage{amsmath}
\usepackage{algorithmic}
\usepackage{algorithm}
\usepackage{multirow}
\usepackage{graphicx}
\usepackage{subfigure}
\usepackage{wrapfig}
\usepackage{amsthm}
\usepackage{multicol}
\usepackage{amssymb}
\usepackage{pdfpages}
\usepackage{makecell}

\allowdisplaybreaks[4]

\newtheorem{proposition}{Proposition}
{\bgroup
  \addtolength\abovedisplayshortskip{#1} 
  \addtolength\abovedisplayskip{#1}
  \addtolength\belowdisplayshortskip{#1}
  \addtolength\belowdisplayskip{#1}}
{\egroup\ignorespacesafterend}
\pdfoutput=1

  \usepackage[nocompress]{cite}
\else
  \usepackage{cite}
\fi
\ifCLASSINFOpdf
\else
\fi
\IEEEpubid{
\fbox{
\parbox{0.975\textwidth}{
\copyright 2021 IEEE. Personal use of this material is permitted. Permission from IEEE must be obtained for all other uses, in any current or future media, including\hfill
reprinting/republishing this material for advertising or promotional purposes, creating new collective works, for resale or redistribution to servers or lists, or reuse of\hfill
any copyrighted component of this work in other works. \vspace{-0.5mm}
}%
}}

\begin{document}
%
\title{Regularizing Deep Networks with Semantic Data Augmentation}

%
%
%
%

\author{Yulin~Wang\IEEEauthorrefmark{1},
        Gao~Huang\IEEEauthorrefmark{1},~\IEEEmembership{Member,~IEEE},
        Shiji~Song,~\IEEEmembership{Senior~Member,~IEEE},
        Xuran~Pan, 
        Yitong~Xia,
        and~Cheng~Wu 
\IEEEcompsocitemizethanks{\IEEEcompsocthanksitem Y. Wang, G. Huang, S. Song, X. Pan and C. Wu are with the Department of Automation, BNRist, Tsinghua University, Beijing 100084, China. Email: \{wang-yl19, pxr18\}@mails.tsinghua.edu.cn, \{gaohuang, shijis, wuc\}@tsinghua.edu.cn. Corresponding author: Gao Huang.
\IEEEcompsocthanksitem Y. Xia is with the School of Automation Science and Electrical Engineering, Beihang University, Beijing, China. Email: xyt990614@buaa.edu.cn
} 
}

%
%

\markboth{IEEE transactions on pattern analysis and machine intelligence}
{Shell \MakeLowercase{\textit{et al.}}: Bare Advanced Demo of IEEEtran.cls for IEEE Computer Society Journals}
%



\IEEEtitleabstractindextext{%
\begin{abstract}
  Data augmentation is widely known as a simple yet surprisingly effective technique for regularizing deep networks. Conventional data augmentation schemes, \emph{e.g.}, flipping, translation or rotation, are low-level, data-independent and class-agnostic operations, leading to limited diversity for augmented samples. To this end, we propose a novel semantic data augmentation algorithm to complement traditional approaches. The proposed method is inspired by the intriguing property that deep networks are effective in learning linearized features, \emph{i.e.}, certain directions in the deep feature space correspond to meaningful semantic transformations, \emph{e.g.}, changing the background or view angle of an object. Based on this observation, translating training samples along many such directions in the feature space can effectively augment the dataset for more diversity. To implement this idea, we first introduce a sampling based method to obtain semantically meaningful directions efficiently. Then, an upper bound of the \emph{expected} cross-entropy (CE) loss on the augmented training set is derived by assuming the number of augmented samples goes to infinity, yielding a highly efficient algorithm. In fact, we show that the proposed \emph{implicit semantic data augmentation (ISDA)} algorithm amounts to minimizing a novel robust CE loss, which adds minimal extra computational cost to a normal training procedure. 
  In addition to supervised learning, ISDA can be applied to semi-supervised learning tasks under the consistency regularization framework, where ISDA amounts to minimizing the upper bound of the expected KL-divergence between the augmented features and the original features. 
  Although being simple, ISDA consistently improves the generalization performance of popular deep models (\emph{e.g.}, ResNets and DenseNets) on a variety of datasets, \emph{i.e.}, CIFAR-10, CIFAR-100, SVHN, ImageNet, and Cityscapes. Code for reproducing our results is available at \color{blue}{\textit{https://github.com/blackfeather-wang/ISDA-for-Deep-Networks}}.

\end{abstract}

}

\maketitle

\begingroup\renewcommand\thefootnote{\IEEEauthorrefmark{1}}
\footnotetext{\emph{Equal contribution.}}

\IEEEdisplaynontitleabstractindextext

%
\IEEEpeerreviewmaketitle


\input{introduction.tex}
\input{related.tex}

\input{method.tex}
\input{experiments.tex}
\input{conclusion.tex}
\ifCLASSOPTIONcompsoc
  \section*{Acknowledgments}
\else
  \section*{Acknowledgment}
\fi

This work is supported in part by the Ministry of Science and Technology of China under Grant 2018AAA0101604, the National Natural Science Foundation of China under Grants 62022048, 61906106 and 61936009, the Institute for Guo Qiang of Tsinghua University and Beijing Academy of Artificial Intelligence.

\ifCLASSOPTIONcaptionsoff
  \newpage
\fi



%



\bibliographystyle{IEEEtran}
\bibliography{IEEEabrv,IEEEtran}

\newpage
\onecolumn
\appendices

\section{Training Details for Supervised Learning}
\label{Training_Details_sup}

On CIFAR, we implement the ResNet, SE-ResNet, Wide-ResNet, ResNeXt and DenseNet.
The SGD optimization algorithm with a Nesterov momentum is applied to train all models. Specific hyper-parameters for training are presented in Table \ref{Training_hp}.

\begin{table*}[h]
    \centering
    \vskip -0.1in
    \caption{Training configurations on CIFAR. `$l_r$' denotes the learning rate.}
    \label{Training_hp}
    \setlength{\tabcolsep}{1.5mm}{
    \vskip -0.1in
    \renewcommand\arraystretch{1.25} 
    \begin{tabular}{c|c|c|c|c|c|c}
    \hline
    Network & Total Epochs & Batch Size & Weight Decay & Momentum & Initial $l_r$ & $l_r$ Schedule \\
    \hline
    ResNet & 160 & 128 & 1e-4 & 0.9 & 0.1 & Multiplied by 0.1 in $80^{th}$ and $120^{th}$ epoch. \\
    \hline
    SE-ResNet & 200 & 128 & 1e-4 & 0.9 & 0.1 & Multiplied by 0.1 in $80^{th}$, $120^{th}$ and $160^{th}$ epoch. \\
    \hline
    Wide-ResNet & 240 & 128 & 5e-4 & 0.9 & 0.1 & Multiplied by 0.2 in $60^{th}$, $120^{th}$, $160^{th}$ and $200^{th}$ epoch. \\
    \hline
    DenseNet-BC & 300 & 64 & 1e-4 & 0.9 & 0.1 & Multiplied by 0.1 in $150^{th}$, $200^{th}$ and $250^{th}$ epoch. \\
    \hline
    ResNeXt & 350 & 128 & 5e-4 & 0.9 & 0.05 & Multiplied by 0.1 in $150^{th}$, $225^{th}$ and $300^{th}$ epoch. \\
    \hline
    Shake Shake &\multirow{1}{*}{1800}&\multirow{1}{*}{64}&\multirow{1}{*}{1e-4}&\multirow{1}{*}{0.9}&\multirow{1}{*}{0.1}&\multirow{1}{*}{Cosine learning rate.} \\
    \hline
    \end{tabular}}
    \vskip -0.1in
\end{table*}

On ImageNet, we train all models for 300 epochs using the same L2 weight decay and momentum as CIFAR. The initial learning rate is set to 0.2 and annealed with a cosine schedule \cite{loshchilov2016sgdr}. The size of mini-batch is set to 512. We adopt $\lambda_0=1$ for DenseNets and $\lambda_0=7.5$ for ResNets and ResNeXts, except for using $\lambda_0=5$ for ResNet-101.

\section{Training Details for Semi-supervised Learning}
\label{Training_Details_semi_sup}
In semi-supervised learning experiments, we implement the CNN-13 network, which has been widely used for semi-supervised learning \cite{luo2018smooth, xie2019unsupervised, verma2019interpolation, laine2016temporal, tarvainen2017mean, miyato2018virtual}. The network architecture of CNN-13 is shown in Table \ref{table:cnn13}. Following \cite{verma2019interpolation}, the network is trained for 400 epochs using the SGD optimization algorithm with a Nesterov momentum. The initial learning rate is set to 0.1 and then annealed to zero with the cosine annealing technique. We use a momentum of 0.9 and a L2 weight decay of 1e-4. The batch size is set to 128.
\begin{table}[h]
  \centering
  \vskip -0.1in
	\caption{\label{table:cnn13}
    The network architecture of CNN-13.}
    \vskip -0.1in
    \setlength{\tabcolsep}{1.5mm}{
      \renewcommand\arraystretch{1.15} 
  \begin{tabular}{c}
		\hline
    Input: $3\times32\times32$ RGB image\\
    \hline
		$3\times3$ conv. $128$ lReLU ($\alpha=0.1$)\\
		$3\times3$ conv. $128$ lReLU ($\alpha=0.1$)\\
    $3\times3$ conv. $128$ lReLU ($\alpha=0.1$)\\
    \hline
    $2\times2$ max-pool, dropout 0.5 \\
    \hline
		$3\times3$ conv. $256$ lReLU ($\alpha=0.1$)\\
		$3\times3$ conv. $256$ lReLU ($\alpha=0.1$)\\
    $3\times3$ conv. $256$ lReLU ($\alpha=0.1$)\\
    \hline
    $2\times2$ max-pool, dropout 0.5 \\
    \hline
		$3\times3$ conv. $512$ lReLU ($\alpha=0.1$)\\
		$1\times1$ conv. $256$ lReLU ($\alpha=0.1$) \\
    $1\times1$ conv. $128$ lReLU ($\alpha=0.1$) \\
    \hline
    $6\times6$ Global average pool\\
    \hline
		Fully connected $128 \to 10$ or $128 \to 100$ softmax\\
		\hline
  \end{tabular}}
  \vskip -0.1in
\end{table}

\section{Hyper-parameter Selection for Baselines}
\label{hyper-para-baseline}
All baselines are implemented with the training configurations mentioned above.
In supervised learning experiments, the dropout rate is set as 0.3 for comparison if it is not applied in the basic model, following the instruction in \cite{Srivastava2014DropoutAS}. For noise rate in disturb label, 0.05 is adopted in Wide-ResNet-28-10 on both CIFAR-10 and CIFAR-100 datasets and ResNet-110 on CIFAR 10, while 0.1 is used for ResNet-110 on CIFAR 100. Focal Loss contains two hyper-parameters $\alpha$ and $\gamma$. Numerous combinations have been tested on the validation set and we ultimately choose $\alpha=0.5$ and $\gamma=1$ for all four experiments. 
For L$_q$ loss, although \cite{Zhang2018GeneralizedCE} states that $q=0.7$ achieves the best performance in most conditions, we suggest that $q=0.4$ is more suitable in our experiments, and therefore adopted. For center loss, we find its performance is largely affected by the learning rate of the center loss module, therefore its initial learning rate is set as 0.5 for the best generalization performance.

For generator-based augmentation methods, we apply the GANs structures introduced in \cite{arjovsky2017wasserstein, mirza2014conditional, odena2017conditional, chen2016infogan} to train the generators. For WGAN, a generator is trained for each class in CIFAR-10 dataset. For CGAN, ACGAN and infoGAN, a single model is simply required to generate images of all classes. A 100 dimension noise drawn from a standard normal distribution is adopted as input, generating images corresponding to their label. Specially, infoGAN takes additional input with two dimensions, which represent specific attributes of the whole training set. Synthetic images are involved with a fixed ratio in every mini-batch. Based on the experiments on the validation set, the proportion of generalized images is set as $1/6$.

In semi-supervised learning experiments, we use exactly the same hyper-parameter settings and searching policies mentioned in the original papers of baselines \cite{laine2016temporal, tarvainen2017mean, miyato2018virtual}. Our implementations achieve the same or sometimes better results than the reported results in their papers under the same experimental settings.

\section{Results on COCO}
\label{COCO_results}
To demonstrate that deep networks learn better representations with ISDA, we employ the models (i.e., ResNet-50) reported in Table \ref{ImageNet Results} as backbones, and finetune them for the object detection task and the instance segmentation task on MS COCO \cite{lin2014microsoft}. We consider three state-of-the-art methods, Faster-RCNN \cite{ren2015faster}, Mask-RCNN \cite{he2017mask} and Cascade-RCNN \cite{cai2018cascade} implemented in the default configuration of MMDetection \cite{mmdetection} with FPN backbones \cite{lin2017feature}. The results of the two tasks are presented in Table \ref{Detection Results} and Table \ref{InstanceSeg} respectively. We find that ISDA contributes to a better initialization than the standard models, improving the performance of different algorithms on both tasks consistently.

\begin{table*}[t]
    \centering
    \caption{
        Object detection results on COCO. We initialize the backbones of several modern detection methods using ResNets pre-trained on ImageNet with and without ISDA. The COCO style box average precision (AP) metric is adopted, where AP$_{50}$ and AP$_{75}$ denote AP over $50\%$ and $75\%$ IoU thresholds, mAP takes the average value of AP over different thresholds ($50\%$--$95\%$), and mAP$_\textnormal{S}$, mAP$_\textnormal{M}$ and mAP$_\textnormal{L}$ denote mAP for objects at different scales. For a fair comparison, we report the baselines using the official pre-trained models provided by PyTorch in the `original' row. All results are presented in percentages (\%). The better results are \textbf{bold-faced}.
    }
    \vskip -0.1in
    \label{Detection Results}
    \setlength{\tabcolsep}{3.5mm}{
    \vspace{5pt}
    \renewcommand\arraystretch{1.21}
    \begin{tabular}{c|c|c|c|c|c|c|c}
    \hline
    Method&Backbone& mAP & AP$_{50}$ & AP$_{75}$ & mAP$_\textnormal{S}$ & mAP$_\textnormal{M}$ & mAP$_\textnormal{L}$\\
    \hline
    \multirow{3}{*}{Faster-RCNN-FPN \cite{ren2015faster}} &  ResNet-50 (Original) &  37.4 & 58.1 & 40.4 & 21.2 & 41.0 & 48.1   \\
    &  ResNet-50 (Our Implementation) &  38.0 & 58.8 & 41.2 & 21.6 & 41.9 & 48.6 \\
    &  ResNet-50 + ISDA &  \textbf{38.6} & \textbf{59.5} & \textbf{41.8} & \textbf{22.8} & \textbf{42.6} & \textbf{48.7} \\
    \cline{2-8}
    \hline
    \multirow{3}{*}{Mask-RCNN-FPN \cite{he2017mask}}&  ResNet-50 (Original) &  38.2 & 58.8 & 41.4 & 21.9 & 40.9 & 49.5   \\
    &  ResNet-50 (Our Implementation) &  38.9 & 59.5 & 42.4 & 22.6 & 42.2 & \textbf{50.1} \\
    &  ResNet-50 + ISDA &  \textbf{39.4} & \textbf{60.1} & \textbf{43.0} & \textbf{23.1} & \textbf{43.4} & {49.8} \\
    \cline{2-8}
    \hline
    \multirow{3}{*}{Cascade-RCNN-FPN \cite{cai2018cascade}}&  ResNet-50 (Original) &  40.3 & 58.6 & 44.0 & 22.5 & 43.8 & 52.9   \\
    &  ResNet-50 (Our Implementation) &  40.8 & 59.2 & 44.4 & 22.8 & 44.4 & 53.1 \\
    &  ResNet-50 + ISDA &  \textbf{41.4} & \textbf{59.7} & \textbf{45.1} &\textbf{23.8} & \textbf{44.9} & \textbf{53.8} \\
    \cline{2-8}
    \hline

    \end{tabular}}
    \vskip -0.15in
\end{table*}

\begin{table*}[t]
    \centering
    \caption{
        Instance segmentation results on COCO. We report the COCO style mask average precision (AP) of Mask-RCNN with different backbones. The baselines using the official pre-trained models provided by PyTorch is reported in the `original' row. All results are presented in percentages (\%). The better results are \textbf{bold-faced}.
    }
    \vskip -0.1in
    \label{InstanceSeg}
    \setlength{\tabcolsep}{4mm}{
    \vspace{5pt}
    \renewcommand\arraystretch{1.21}
    \begin{tabular}{c|c|c|c|c|c|c|c}
    \hline
    Method&Backbone& mAP & AP$_{50}$ & AP$_{75}$ & mAP$_\textnormal{S}$ & mAP$_\textnormal{M}$ & mAP$_\textnormal{L}$\\
    \hline
    \multirow{3}{*}{Mask-RCNN-FPN \cite{he2017mask}}&  ResNet-50 (Original) &  34.7 & 55.7 & 37.2 & 18.3 & 37.4 & 47.2   \\
    &  ResNet-50 (Our Implementation) &  35.2 & 56.5 & 37.6 & 18.8 & 38.4 & \textbf{47.5} \\
    &  ResNet-50 + ISDA &  \textbf{35.6} & \textbf{56.7} & \textbf{38.0} & \textbf{19.2} & \textbf{39.2} & {47.1} \\
    \hline
    \end{tabular}}
    \vskip -0.15in
\end{table*}

\section{Visualization of Augmented Samples}
\label{reverse_alg}


To explicitly demonstrate the semantic changes generated by ISDA, we propose an algorithm to map deep features back to the pixel space. Some extra visualization results are shown in Figure \ref{Extra}.

An overview of the algorithm is presented in Figure \ref{Reversing}.
As there is no closed-form inverse function for convolutional networks like ResNet or DenseNet, the mapping algorithm acts in a similar way to \cite{mahendran2015understanding} and \cite{Upchurch2017DeepFI}, by fixing the model and adjusting inputs to find images corresponding to the given features. However, given that ISDA augments semantics of images in essence, we find it ineffective to directly optimize the inputs in the pixel space. Therefore, we add a fixed pre-trained generator $\mathcal{G}$, which is obtained from training a wasserstein GAN \cite{arjovsky2017wasserstein} on CIFAR and the pre-trained BigGAN \cite{brock2018large} on ImgeNet\footnote{https://github.com/ajbrock/BigGAN-PyTorch}, to produce images for the classification model, and optimize the inputs of the generator instead. This approach makes it possible to effectively reconstruct images with augmented semantics.

\begin{figure*}[t]
  \begin{center}
  \centerline{\includegraphics[width=0.9\columnwidth]{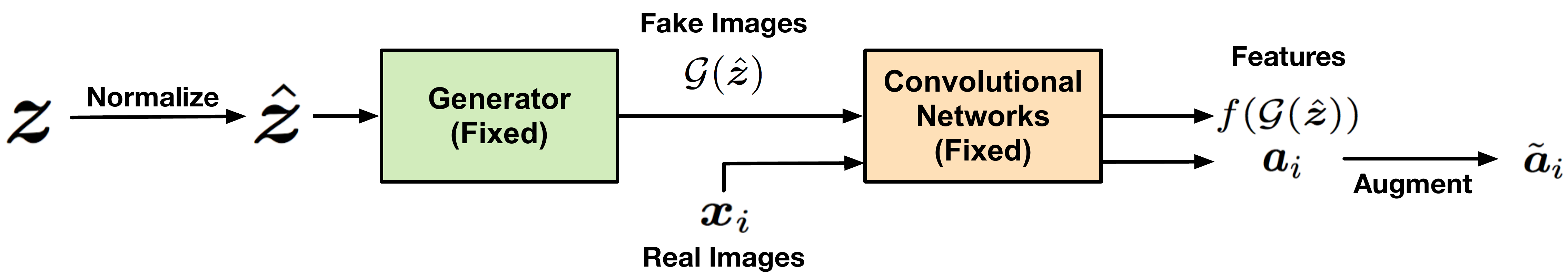}}
  \vskip -0.1in
  \caption{Overview of the reverse mapping algorithm. We adopt a fixed generator $\mathcal{G}$ obtained by training a wasserstein GAN to generate fake images for convolutional networks, and optimize the inputs of $\mathcal{G}$ in terms of the consistency in both the pixel space and the deep feature space.}
  \label{Reversing}
  \end{center}
  \vskip -0.1in
\end{figure*}


\begin{figure*}
  \begin{center}
  \includegraphics[width=0.8\columnwidth]{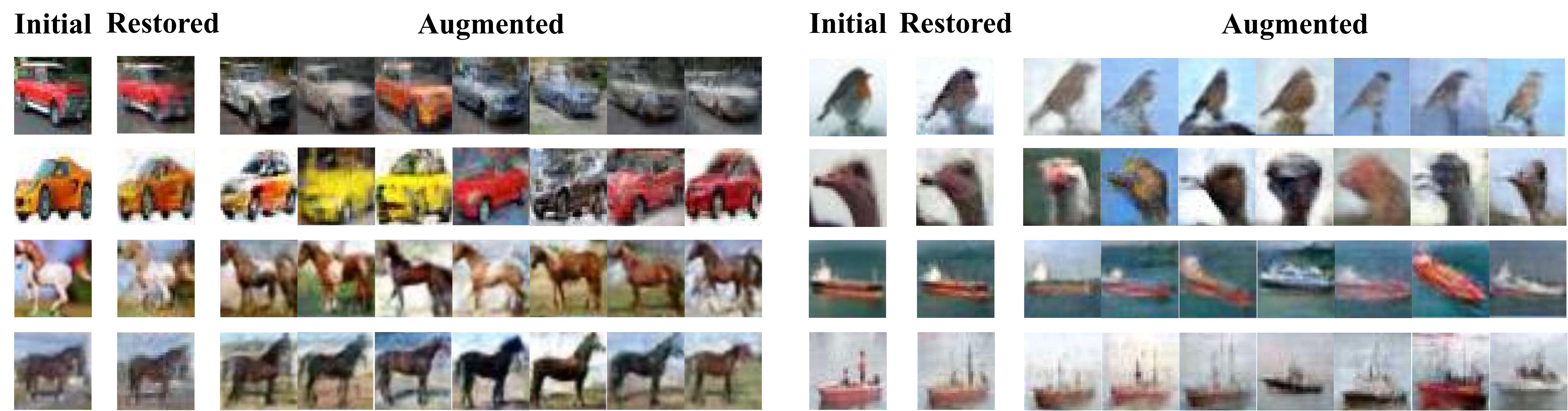}
  \caption{Visualization of the semantically augmented images on CIFAR. ISDA is able to alter the semantics of images that are unrelated to the class identity, like backgrounds, colors of skins, type of cars, etc.}
  \vskip -0.1in
  \label{Extra}
  \end{center}
  \vskip -0.1in
\end{figure*}

\begin{figure*}[!t]
  \begin{center}
  \includegraphics[width=\columnwidth]{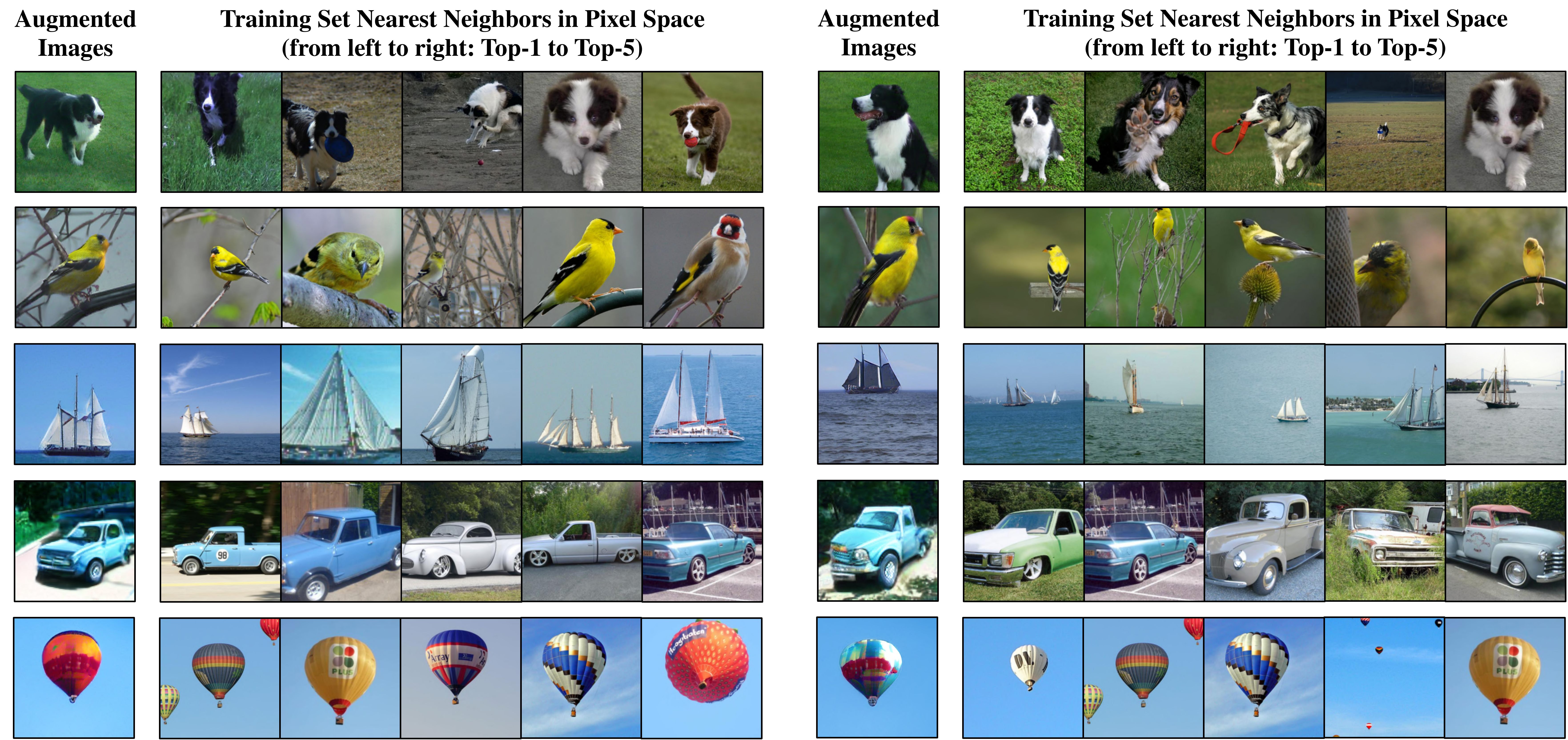}
  \caption{Training set nearest neighbors of the augmented images in pixel space on ImageNet.}
  \vskip -0.1in
  \label{NN}
  \end{center}
  \vskip -0.1in
\end{figure*}
The mapping algorithm can be divided into two steps:

\textbf{Step I. }Assume a random variable $\bm{z}$ is normalized to $\hat{\bm{z}}$ and input to $\mathcal{G}$, generating fake image $\mathcal{G}(\hat{\bm{z}})$. $\bm{x}_{i}$ is a real image sampled from the dataset (such as CIFAR). $\mathcal{G}(\hat{\bm{z}})$ and $\bm{x}_{i}$ are forwarded through a pre-trained convolutional network to obtain deep feature vectors $f(\mathcal{G}(\hat{\bm{z}}))$ and $\bm{a}_{i}$. The first step of the algorithm is to find the input noise variable $\bm{z}_{i}$ corresponding to $\bm{x}_{i}$, namely 
\begin{equation}
    \label{ra1}
    \bm{z}_{i} = \arg\min_{\bm{z}} \|f(\mathcal{G}(\hat{\bm{z}})) - \bm{a}_{i}\|_{2}^{2} +
    \eta\|\mathcal{G}(\hat{\bm{z}}) - \bm{x}_{i}\|_{2}^{2},\
    s.t.\ \hat{\bm{z}} = \frac{\bm{z} - \overline{\bm{z}}}{std(\bm{z})},
\end{equation}
where $ \overline{\bm{z}}$ and $std(\bm{z})$ are the average value and the standard deviation of $\bm{z}$, respectively.
The consistency of both the pixel space and the deep feature space are considered in the loss function, and we introduce a hyper-parameter $\eta$ to adjust the relative importance of two objectives.

\textbf{Step II. }We augment $\bm{a}_{i}$ with ISDA, forming $\tilde{\bm{a}}_{i}$ and reconstructe it in the pixel space. Specifically, we search for $\bm{z}_{i}'$ corresponding to $\tilde{\bm{a}}_{i}$ in the deep feature space, with the start point $\bm{z}_{i}$ found in Step I:
\begin{equation}
    \label{ra2}
    \bm{z}_{i}' = \arg\min_{\bm{z'}} \|f(\mathcal{G}(\hat{\bm{z}}')) - \tilde{\bm{a}}_{i}\|_{2}^{2},\
    s.t.\ \hat{\bm{z}}' = \frac{\bm{z'} - \overline{\bm{z'}}}{std(\bm{z'})}.
\end{equation}
As the mean square error in the deep feature space is optimized to nearly 0, $\mathcal{G}(\hat{\bm{z}_{i}}')$
is supposed to represent the image corresponding to $\tilde{\bm{a}}_{i}$.

On CIFAR, the proposed algorithm is performed on a single batch. In practice, a ResNet-32 network is used as the convolutional network. We solve Eq. (\ref{ra1}), (\ref{ra2}) with the standard gradient descent (GD) algorithm of 10,000 iterations. The initial learning rate is set as 10 and 1 for Step I and Step II respectively, and is divided by 10 every 2,500 iterations. We apply a momentum of 0.9 and a L2 weight decay of 1e-4. 

In addition, to illustrate that the GAN model does not trivially ``remember'' the samples in the training set, we present the Top-5 nearest neighbors of the augmented images in pixel space on ImageNet. One can observe that the augmented samples and the nearest neighbors are visually distinct, suggesting that the former are not simply obtained via memorizing training data.

\end{document}

%% file: introduction.tex
\section{Introduction}

\IEEEPARstart{D}{ata} augmentation is an effective technique to alleviate the overfitting problem in training deep networks \cite{krizhevsky2009learning,krizhevsky2012imagenet,2014arXiv1409.1556S,He_2016_CVPR,2016arXiv160806993H}.
In the context of image recognition, this usually corresponds to applying content preserving transformations, e.g., cropping, horizontal mirroring, rotation and color jittering, on the input samples. 
Although being effective, these augmentation techniques are not capable of performing semantic transformations, such as changing the background of an object or the texture of a foreground object. Recent work has shown that data augmentation can be more powerful if these (class identity preserving) semantic transformations are allowed \cite{NIPS2017_6916, bowles2018gan, antoniou2017data}. For example, by training a generative adversarial network (GAN) for each class in the training set, one could then sample an infinite number of samples from the generator. Unfortunately, this procedure is computationally intensive because training generative models and inferring them to obtain augmented samples are both nontrivial tasks. Moreover, due to the extra augmented data, the training procedure is also likely to be prolonged.

\begin{figure}[t]
    \vskip -0.05in
    \begin{center}
    \centerline{\includegraphics[width=\columnwidth]{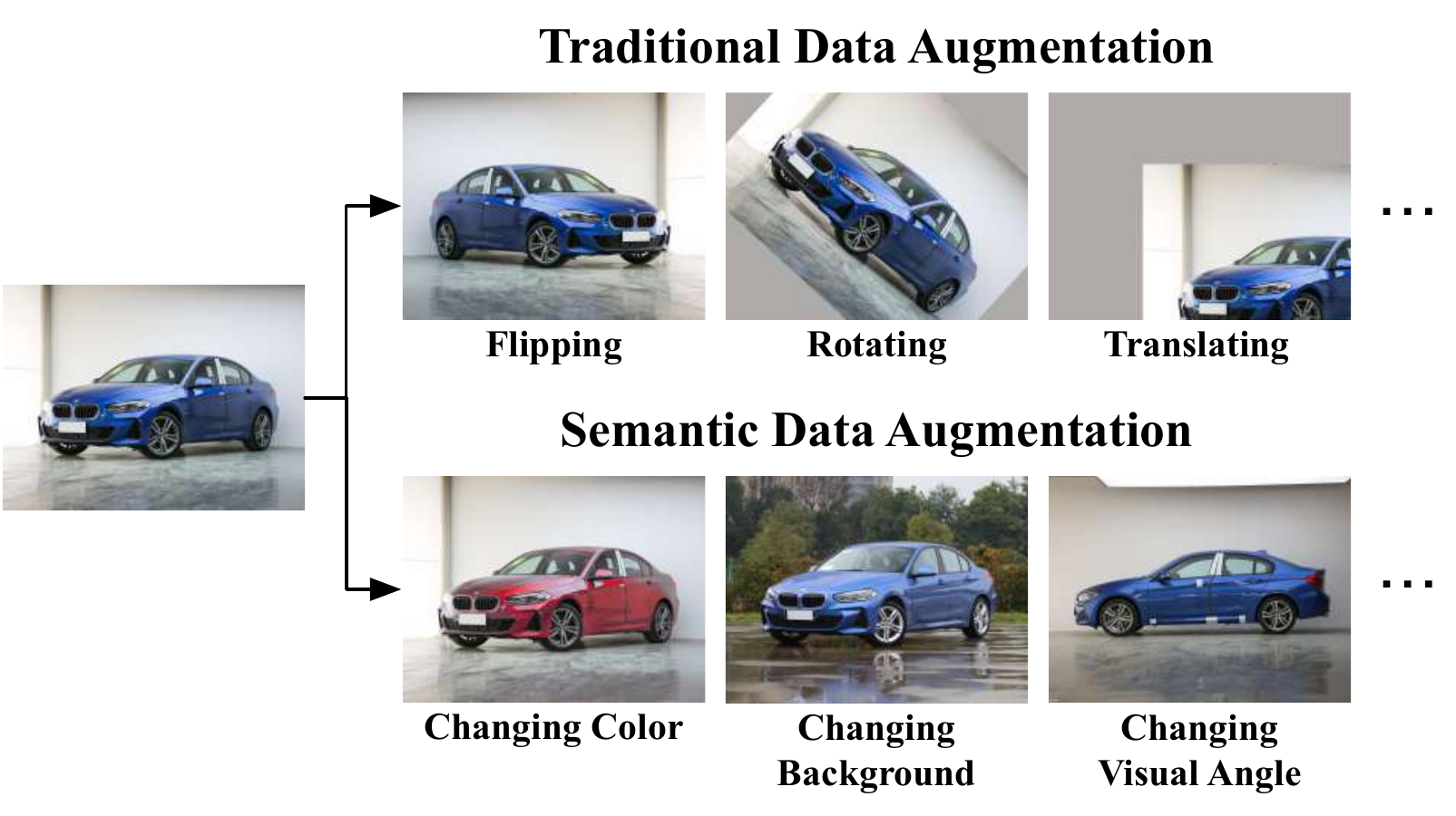}}  
    \vspace{-3ex}
    \caption{The comparison of traditional and semantic data augmentation. Conventionally, data augmentation usually corresponds to naive image transformations (like flipping, rotating, translating, etc.) in the pixel space. Performing class identity preserving semantic transformations (like changing the color of a car,  changing the background of an object, etc.) is another effective approach to augment the training data, which is complementary to traditional techniques.}
    \label{illustration_SDA}
    \end{center}
    \vskip -0.32in
\end{figure}


\begin{figure*}[t]
    \begin{center}
    \centerline{\includegraphics[width=2\columnwidth]{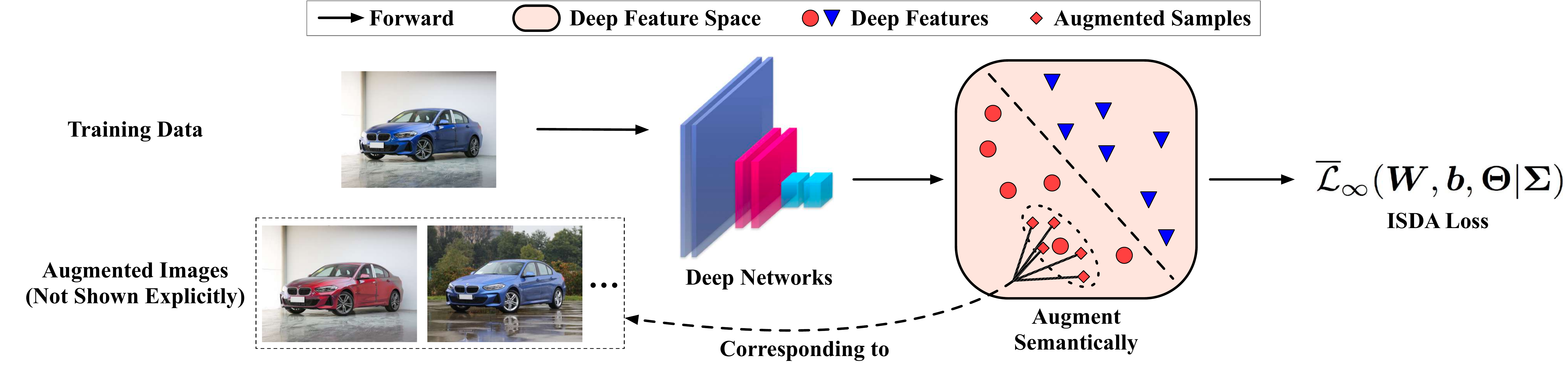}}
    \vskip -0.15in
    \caption{An overview of ISDA. Inspired by the observation that certain directions in the feature space correspond to meaningful semantic transformations, we augment the training data semantically by translating their features along these semantic directions, without involving auxiliary deep networks. The directions are obtained by sampling random vectors from a zero-mean normal distribution with dynamically estimated class-conditional covariance matrices. 
    In addition, instead of performing augmentation explicitly, ISDA boils down to minimizing a closed-form upper-bound of the expected cross-entropy loss on the augmented training set, which makes our method highly efficient. 
    }
    \label{overview}
    \end{center}
    \vskip -0.2in
\end{figure*}

In this paper, we propose an implicit semantic data augmentation (ISDA) algorithm for training deep networks. The ISDA is highly efficient as it does not require training/inferring auxiliary networks or explicitly generating extra training samples. Our approach is motivated by the intriguing observation made by recent work showing that the features deep in a network are usually linearized \cite{Upchurch2017DeepFI, bengio2013better}. Specifically, there exist many semantic directions in the deep feature space, such that translating a data sample in the feature space along one of these directions results in a feature representation corresponding to another sample with the same class identity but different semantics. For example, a certain direction corresponds to the semantic translation of "make-bespectacled". When the feature of a person, who does not wear glasses, is translated along this direction, the new feature may correspond to the same person but with glasses (The new image can be explicitly reconstructed using proper algorithms as shown in \cite{Upchurch2017DeepFI}). Therefore, by searching for many such semantic directions, we can effectively augment the training set in a way complementary to traditional data augmenting techniques.

However, explicitly finding semantic directions is not a trivial task, which usually requires extensive human annotations \cite{Upchurch2017DeepFI}. In contrast, sampling directions randomly is efficient but may result in meaningless transformations. For example, it makes no sense to apply the "make-bespectacled" transformation to the ``car'' class. In this paper, we adopt a simple method that achieves a good balance between effectiveness and efficiency. In specific, we perform an online estimate of the covariance matrix of the features for \emph{each} class, which captures the intra-class variations. Then we sample directions from a zero-mean multi-variate normal distribution with the estimated covariance and apply them to the features of training samples in that class to augment the dataset. In this way, the chance of generating meaningless semantic transformations can be significantly reduced.

To further improve the efficiency, we derive a closed-form upper bound of the \emph{expected} cross-entropy (CE) loss with the proposed data augmentation scheme. Therefore, instead of performing the augmentation procedure explicitly, we can directly minimize the upper bound, which is, in fact, a novel robust surrogate loss function. As there is no need to generate explicit data samples, we call our algorithm \emph{implicit semantic data augmentation (ISDA)}. Compared to existing semantic data augmentation algorithms, the proposed ISDA is able to be conveniently implemented on top of most deep models without introducing auxiliary models or noticeable extra computational cost.

In addition to supervised learning tasks, we further apply the proposed ISDA algorithm to more realistic semi-supervised learning scenarios, where only a small subset of all available training data are associated with labels \cite{rasmus2015semi, kingma2014semi,miyato2018virtual, tarvainen2017mean, laine2016temporal}. For samples with labels, we simply minimize the aforementioned upper bound as the surrogate loss. For unlabeled samples, as it is unable to obtain the surrogate loss for ISDA directly, we propose to enforce their semantic consistency. To be specific, since ISDA performs class identity preserving semantic transformations, which should not affect the model prediction on categories, we augment the deep features of unlabeled data, and minimize the KL-divergence between the predictions of the augmented features and the original features. Similarly, an upper bound of the expected KL-divergence is derived as the optimization objective. ISDA can be implemented together with state-of-the-art deep semi-supervised learning algorithms and significantly improves their performance.




Although being simple, the proposed ISDA algorithm is surprisingly effective. Extensive empirical evaluations including supervised / semi-supervised image classification on CIFAR, SVHN and ImageNet and semantic segmentation on Cityscapes are conducted. Results show that ISDA consistently improves the generalization performance of popular deep networks and enable the models to learn better representations. 






Parts of the results in this paper were published originally in its conference version \cite{wang2019implicit}. However, this paper extends our earlier work in several important aspects:
\begin{itemize}
    \item We extend the proposed ISDA algorithm to deep semi-supervised learning, and empirically validate it on widely used image classification benchmarks (Section \ref{semi_ISDA} \& \ref{semi_ISDA_result}).
    \item We present more results on ImageNet (Table \ref{ImageNet Results}) with different deep networks (i.e. ResNets, ResNeXts and DenseNets). 
    \item We further apply our algorithm to the semantic segmentation task on Cityscapes (Table \ref{Segmentation}), and report positive results.
    \item An analysis of the computational complexity is given (Section \ref{Complexity}), showing that ISDA introduces negligible computational overhead theoretically. We also report the additional time consumption of ISDA in practice (Section \ref{Computational_Cost_ISDA}).
    \item Additional analytical results including a visualization on ImageNet, a t-SNE visualization of deep features, a sensitivity test and an analysis of the tightness of the upper bound are presented (Section \ref{Analytical_results}).
\end{itemize}

%% file: related.tex
\section{Related Work}

In this section, we briefly review existing research on related topics.

\textbf{Data augmentation} is a widely used technique to regularize deep networks.
For example, in image recognition tasks, augmentation methods like random flipping, mirroring and rotation are applied to enforce the geometric invariance of convolutional networks \cite{He_2016_CVPR, 2016arXiv160806993H, 2014arXiv1409.1556S, srivastava2015training}. These classic techniques are fundamental to obtain highly generalized deep models.
It is shown in some literature that abandoning certain information in training images is also an effective approach to augment the training data. Cutout \cite{devries2017improved} and random erasing \cite{zhong2017random} randomly cut a rectangle region of an input image out to perform augmentation. In addition, several studies focus on automatic data augmentation techniques. E.g., AutoAugment \cite{2018arXiv180509501C} is proposed to search for a better augmentation strategy among a large pool of candidates using reinforcement learning. A key concern on AutoAugment is that the searching algorithm suffers from extensive computational and time costs. Similar to our method, learning with marginalized corrupted features \cite{maaten2013learning} can be viewed as an implicit data augmentation technique, but it is limited to simple linear models. Feature transfer learning \cite{yin2019feature} explicitly augments the under-represented data in the feature space, while it merely focuses on the imbalance face images. Complementarily, recent research shows that semantic data augmentation techniques which apply class identity preserving transformations (e.g. changing backgrounds of objects or varying visual angles) to the training data are effective as well \cite{jaderberg2016reading, bousmalis2017unsupervised, NIPS2017_6916, antoniou2017data}. This is usually achieved by generating extra semantically transformed training samples with specialized deep structures such as DAGAN \cite{antoniou2017data}, domain adaptation networks \cite{bousmalis2017unsupervised}, or other GAN-based generators \cite{jaderberg2016reading, NIPS2017_6916}. Although being effective, these approaches are nontrivial to implement and computationally expensive, due to the need to train generative models beforehand and infer them during training. 


\textbf{Robust loss function. }
As shown in the paper, ISDA amounts to minimizing a novel robust loss function. Therefore, we give a brief review of related work on this topic. 
Recently, several robust loss functions are proposed to improve the generalization performance of deep networks.
For example, the L$_q$ loss \cite{Zhang2018GeneralizedCE} is a balanced form between the cross entropy (CE) loss and mean absolute error (MAE) loss, derived from the negative Box-Cox transformation. It is designed to achieve the robustness against corrupted labels in the training set, but also achieves effective performance to improve the generalization performance. Focal loss \cite{Lin2017FocalLF} attaches high weights to a sparse set of hard examples to prevent the vast number of easy samples from dominating the training of the network. The idea of introducing a large decision margin for CE loss has been studied in \cite{liu2016large, Liang2017SoftMarginSF, Wang2018EnsembleSS}. These researches propose to maximize the cosine distance between deep features of samples from different classes, in order to alleviate the overfitting brought by the distribution gap between the training data and the real distribution. In \cite{Sun2014DeepLF}, the CE loss and the contrastive loss are combined to learn more discriminative features. From a similar perspective, center loss \cite{wen2016discriminative} simultaneously learns a center for deep features of each class and penalizes the distances between the samples and their corresponding class centers in the feature space, enhancing the intra-class compactness and inter-class separability.

\textbf{Semantic transformations via deep features. }
Our work is inspired by the fact that high-level representations learned by deep convolutional networks can potentially capture abstractions with semantics \cite{Bengio2009DeepArchitechture, bengio2013better}. 
In fact, translating deep features along certain directions has been shown to be corresponding to performing meaningful semantic transformations on the input images. For example, deep feature interpolation \cite{Upchurch2017DeepFI} leverages linear interpolations of deep features from a pre-trained neural network to edit the semantics of images. Variational Autoencoder (VAE) and Generative Adversarial Network (GAN) based methods \cite{Choi2018StarGANUG, zhu2017unpaired, He2018AttGANFA} establish a latent representation corresponding to the abstractions of images, which can be manipulated to perform semantic transformations. Generally, these methods reveal that there exist some semantically meaningful directions in the deep feature space, which can be leveraged to perform semantic data augmentation efficiently.

\textbf{Uncertainty modeling. }
Similar to us, some of previous works on deep learning with uncertainty \cite{kendall2017uncertainties, gal2015bayesian, gal2016dropout} also assume a Gaussian distribution for the deep feature or prediction of each sample. For example, in the context of face recognition and person re-identification, the probabilistic representations are leveraged to address the issues of ambiguous faces \cite{shi2019probabilistic} and data outliers/label noises \cite{yu2019robust}. In multi-task learning, the homoscedastic task uncertainty is used to learn the weights of different tasks \cite{kendall2018multi}. This technique is also exploited to object detection to model the uncertainty of bounding boxes \cite{he2019bounding}. Given that the proposed ISDA algorithm aims at augmenting training data semantically, our motivation is fundamentally different from these works. In addition, ISDA involves novel techniques such as estimating class-conditional covariance matrices and the derivation of the surrogate loss.


\textbf{Deep semi-supervised learning. }
Since ISDA can also be applied to semi-supervised learning tasks, we also briefly review recent work in this field. For modern deep learning, the precise annotations of a sufficiently large training set are usually expensive and time-consuming to acquire. To save the cost of annotation, a nice solution is training models on a small set of labeled data together with a large number of unlabeled samples, which is named semi-supervised learning.
The main methods on this topic can be divided into two sorts, teacher-based methods and perturbation-based methods. The former establish a `teacher model' to provide supervision for unlabeled data. 
For example, temporal ensemble \cite{laine2016temporal} uses the moving averaged prediction of the model on unannotated samples as pseudo labels. Mean teacher \cite{tarvainen2017mean} performs an exponential moving average on the parameters of models to obtain a teacher network. On the other hand, perturbation-based methods add small perturbations to the input images and enforce the prediction consistency of the network between the perturbed images and original images. VAT \cite{miyato2018virtual} proposes to apply adversarial perturbations. $\Pi$-model \cite{laine2016temporal} minimizes the mean-square distance between the same image with different augmentation schemes. As an augmentation technique, the proposed semi-supervised ISDA algorithm is complementary to both the two types of methods.

%% file: method.tex
\section{Semantic Transformations in Deep Feature Space}
\label{Semantic Transformations in Deep Feature Space}
Deep networks have been known to excel at extracting high-level representations in the deep feature space \cite{He_2016_CVPR, 2016arXiv160806993H, Upchurch2017DeepFI, ren2015faster}, where the semantic relationships between samples can be captured by the spatial positions of their deep features \cite{bengio2013better}. It has been shown in previous work that translating deep features towards certain directions corresponds to meaningful semantic transformations when the features are mapped back to the input space \cite{Upchurch2017DeepFI,Li2016ConvolutionalNF, bengio2013better}. As a matter of fact, such an observation can be leveraged to edit the semantics of images without the help of deep architectures. An example is shown in Figure \ref{linearizing}. Consider feeding an image of a blue car into the deep network and obtaining its deep feature. Then if we translate the deep feature along the directions corresponding to `change-color' or `change-background', we will get the deep features corresponding to the images of the same car but with a red paint or under a different background. 

\begin{figure}[t]
    \begin{center}
    \centerline{\includegraphics[width=\columnwidth]{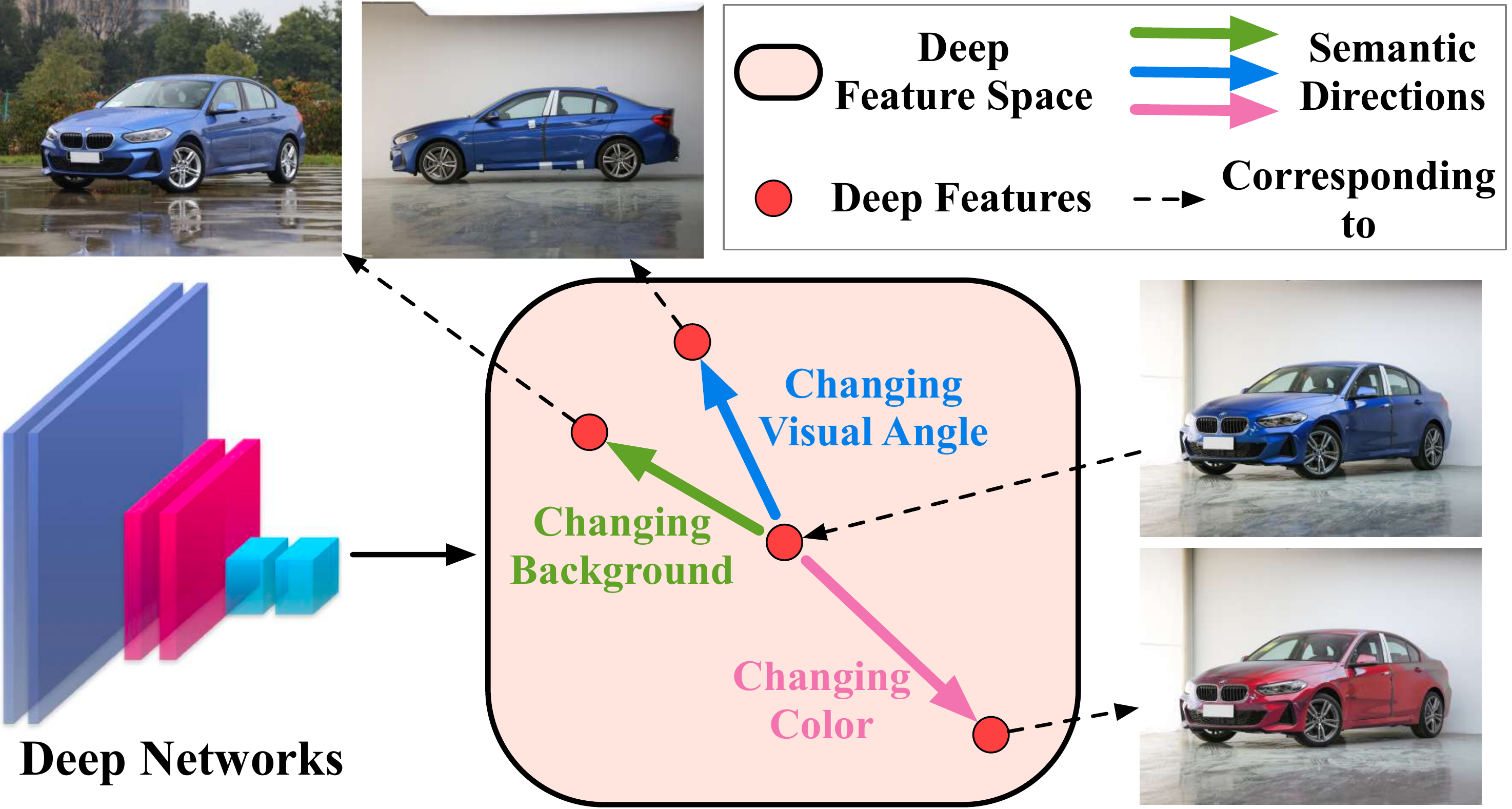}}  
    \vspace{-1ex}
    \caption{
        An illustration of the insight from deep feature interpolation \cite{Upchurch2017DeepFI} and other existing works \cite{Li2016ConvolutionalNF, bengio2013better}, which inspires our method. Transformations like `changing the color of the car' or `changing the background of the image' can be realized by linearly translating the deep features towards the semantic directions corresponding to these transformations.}
    \label{linearizing}
    \end{center}
    \vskip -0.2in
\end{figure}

\begin{figure*}[t]
    \vskip 0.1in
    \begin{center}
    \centerline{\includegraphics[width=2\columnwidth]{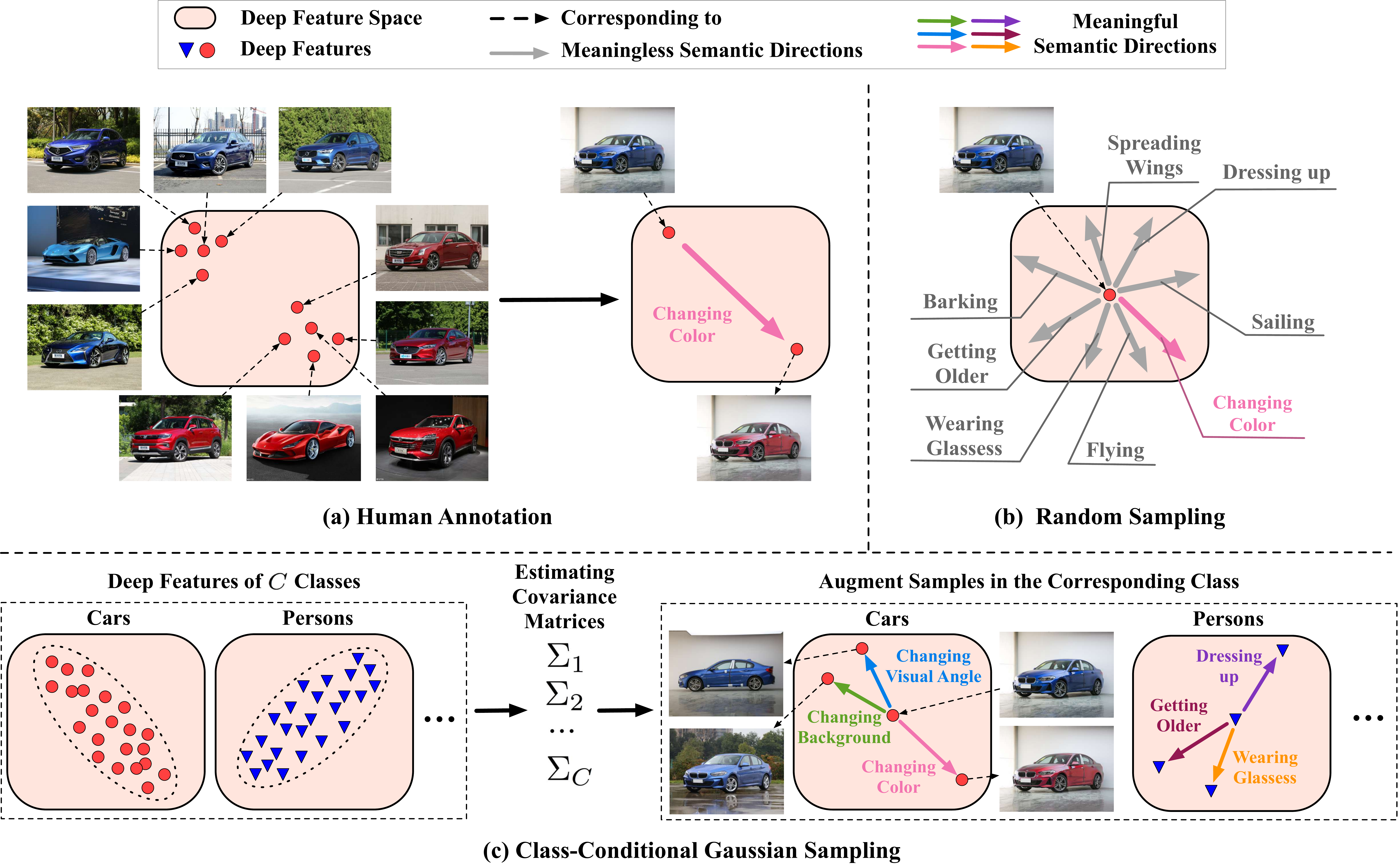}}
    \vskip -0.15in
    \caption{Three different ways to obtain semantic directions for augmentation in the deep feature space. Human annotation is the most precise way. But it requires collecting annotated images for each transformation of each class in advance, which is expensive and time-consuming. In addition, it will inevitably omit potential augmentation transformations. In contrast, finding semantic directions by random sampling is highly efficient, but yields a large number of meaningless transformations. To achieve a nice trade-off between effectiveness and efficiency, we propose to estimate a covariance matrix for the deep features of each class, and sample semantic directions from a zero-mean normal distribution with the estimated class-conditional covariance matrix. The covariance matrix captures the intra-class feature distribution of the training data, and therefore contains rich information of potential semantic transformations.
    }
    \label{Semantic_directions}
    \end{center}
    \vskip -0.225in
\end{figure*}

Based on this intriguing property, we propose to directly augment the semantics of training data by translating their corresponding deep features along many meaningful semantic directions. 
Our method is highly efficient compared with traditional approaches of performing semantic augmentation. Conventionally, to achieve semantic changes, one needs to train, deploy and infer deep generators such as cycle-GAN \cite{zhu2017unpaired} or W-GAN \cite{arjovsky2017wasserstein}. The procedure is both computationally expensive and time-consuming. In contrast, translating deep features just introduces the negligible computational cost of linear interpolation. 

One may challenge that although semantic transformations are efficient to be realized via deep features, showing the results out in the pixel space is difficult \cite{Upchurch2017DeepFI}. However, our goal is not to edit the semantic contents and obtain the results, but to train deep networks with these semantically altered images for the data augmentation purpose. 
Since the augmented features can be directly used for training, it is not necessary to explicitly show the semantic transformations we perform. 
In the following, we will show that our method integrates the augmentation procedure into the training process of deep networks.

\section{Implicit Semantic Data Augmentation (ISDA)}

As aforementioned, certain directions in the deep feature space correspond to meaningful semantic transformations like `make-bespectacled' or `change-visual-angle'. By leveraging this observation, we propose an implicit semantic data augmentation (ISDA) approach to augment the training set semantically via deep features.
Our method has two important components, i.e., online estimation of class-conditional covariance matrices and optimization with a robust loss function. The first component aims to find a distribution from which we can sample meaningful semantic transformation directions for data augmentation, while the second saves us from explicitly generating a large amount of extra training data, leading to remarkable efficiency compared to existing data augmentation techniques.

\subsection{Semantic Direction Sampling}
\label{Semantic Directions Sampling}
A challenge faced by our method is how to obtain suitable semantic directions for augmentation. The directions need to correspond to the semantic transformations that are meaningful for the main object in the image, while do not change the class identity of the image. For example, transformations like wearing glasses or dressing up are suitable to augment the images of persons, but others like flying or sailing are meaningless. In addition, it is obvious that persons in the images should not be transformed to horses or other objects that do not belong to their original class.

Previous work \cite{Upchurch2017DeepFI} proposes to find semantic directions by human annotation. Their method is shown in Figure \ref{Semantic_directions} (a). Take changing the color of a car from blue to red for example. Firstly, they collect two sets of images of blue cars and red cars, respectively, and fed them into deep networks to obtain their deep features. Then they take the vector from the average feature of blue cars to the average feature of red cars. The vector corresponds to the transformation of `changing the color of the car from blue to red'. At last, for a new image to transform, they translate its deep feature along the vector, and map the feature back to the pixel space.
It has been shown that their method is able to perform the specific transformation precisely \cite{Upchurch2017DeepFI}.
Whereas, human annotation is not a feasible approach in the context of semantic data augmentation. For one thing, one needs to collect sufficient annotated images for each possible transformation of each class. This procedure is inefficient. For another, it is difficult to pre-define all possible semantic transformations for each class. The omission will lead to inferior performance.

In terms of efficiency, a possible solution is to obtain semantic directions by random sampling. However, since the deep feature space is highly sparse (e.g., ResNets \cite{He_2016_CVPR} generate 64-dimensional features on CIFAR. Even if each dimension has two possible values, there will be $2^{64}$ possible features.), sampling totally at random will yield many meaningless semantic directions. As shown in Figure \ref{Semantic_directions} (b), transformations like `getting older' or `flying' may be performed for a car. 

To achieve a nice trade-off between the effectiveness and efficiency, we propose to approximate the procedure of human annotation by sampling random vectors from a zero-mean normal distribution with a covariance that is proportional to the intra-class covariance matrix of samples to be augmented. 
The covariance matrix captures the variance of samples in that class and is thus likely to contain rich semantic information. Intuitively, features of the \emph{person} class may vary along the `wearing glasses' direction, as both images of persons with glasses and persons without glasses are contained in the training set. In contrast, the variance along the `having propeller' direction will be nearly zero as all persons do not have propellers. Similarly, features of the \emph{plane} class may vary along the `having propeller' direction, but will have nearly zero variance along the `wearing glasses' direction.
We hope that directions corresponding to meaningful transformations for each class are well represented by the principal components of the covariance matrix of that class.
In addition to its efficiency, the proposed approach can actually leverage more potential semantic transformations than human annotation, as the obtained semantic directions are continuously distributed in the deep feature space.

Consider training a deep network $G$ with weights $\bm{\Theta}$ on a training set 
$\mathcal{D} = \{(\bm{x}_{i}, y_{i})\}$, where $y_{i} \in \{ 1, \ldots, C \}$ is the label of the $i^{\textnormal{th}}$ sample $\bm{x}_{i}$ over $C$ classes.  Let the $A$-dimensional vector $\bm{a}_{i} = [a_{i1}, \ldots, a_{iA}]^{\textnormal{T}} = G(\bm{x}_{i}, \bm{\Theta})$ denote the deep feature of $\bm{x}_{i}$ learned by $G$, and $a_{ij}$ indicate the $j^{\textnormal{th}}$ element of $\bm{a}_{i}$. 


To obtain semantic directions to augment $\bm{a}_{i}$, we establish a zero-mean multi-variate normal distribution $\mathcal{N}(0, \Sigma_{y_i})$, where $\Sigma_{y_i}$ is the class-conditional covariance matrix estimated from the features of all the samples in class $y_i$. In implementation, the covariance matrix is computed in an online fashion by aggregating statistics from all mini-batches. 
Formally, the online estimation algorithm for the covariance matrices is given by:
\begin{equation}
    \label{ave}
    \bm{\mu}_j^{(t)} = \frac{n_j^{(t-1)}\bm{\mu}_j^{(t-1)} + m_j^{(t)} {\bm{\mu}'}_j^{(t)}}
    {n_j^{(t-1)} +m_j^{(t)}},
\end{equation}
\vskip -0.1in
\begin{equation}
    \label{cv}
    \begin{split}
        \Sigma_j^{(t)} 
         = &\frac{n_j^{(t-1)}\Sigma_j^{(t-1)}\!+\!m_j^{(t)} {\Sigma'}_j^{(t)}}
        {n_j^{(t-1)} +m_j^{(t)}}   \\
        & + \frac{n_j^{(t-1)}m_j^{(t)} (\bm{\mu}_j^{(t-1)}\!-\!{\bm{\mu}'}_j^{(t)})
        (\bm{\mu}_j^{(t-1)}\!-\!{\bm{\mu}'}_j^{(t)})^{\textnormal{T}}}
        {(n_j^{(t-1)} +m_j^{(t)})^2},
    \end{split}
\end{equation}
\begin{equation}
    \label{sum}
    n_j^{(t)} = n_j^{(t-1)} + m_j^{(t)},
\end{equation}
where $\bm{\mu}_j^{(t)}$ and $\Sigma_j^{(t)}$ are the estimates of average values and covariance matrices of the features of $j^{\textnormal{th}}$ class at $t^{\textnormal{th}}$ step. ${\bm{\mu}'}_j^{(t)}$ and ${\Sigma'}_j^{(t)}$ are the average values and covariance matrices of the features of $j^{\textnormal{th}}$ class in $t^{\textnormal{th}}$ mini-batch. $n_j^{(t)}$ denotes the total number of training samples belonging to $j^{\textnormal{th}}$ class in all $t$ mini-batches,
and $m_j^{(t)}$ denotes the number of training samples belonging to $j^{\textnormal{th}}$ class only in $t^{\textnormal{th}}$ mini-batch.

During training, $C$ covariance matrices are computed, one for each class. The augmented feature $\tilde{\bm{a}}_{i}$ is obtained by translating $\bm{a}_{i}$ along a random direction sampled from $\mathcal{N}(0, \lambda\Sigma_{y_i})$. Equivalently, we have:
\begin{equation}
    \tilde{\bm{a}}_{i} \sim \mathcal{N}(\bm{a}_{i}, \lambda\Sigma_{y_i}),
\end{equation}
where $\lambda$ is a positive coefficient to control the strength of semantic data augmentation. As the covariances are computed dynamically during training, the estimation in the first few epochs is not quite informative when the network is not well trained. To address this issue, we let $\lambda  = (t/T)\!\times\!\lambda_0$ be a function of the current iteration $t$, thus to reduce the impact of the estimated covariances on our algorithm early in the training stage.

\subsection{Upper Bound of the Expected Loss}
\label{sec_4_2}
A naive method to implement the semantic data augmentation is to explicitly augment each $\bm{{a}}_{i}$ for $M$ times, forming an augmented feature set $\{(\bm{a}_{i}^{1}, y_{i}), \ldots, (\bm{a}_{i}^{M}, y_{i})\}_{i=1}^{N}$ of size $MN$, where $\bm{a}_{i}^{m}$ is $m^{\textnormal{th}}$ sample of augmented features for sample $\bm{x}_i$. 
Then the networks are trained by minimizing the cross-entropy (CE) loss:
\begin{equation}
    \label{eq2}
    \mathcal{L}_{M}(\bm{W}, \bm{b}, \bm{\Theta})\!=\!
    \frac{1}{N}\!
    \sum_{i=1}^{N}\!\frac{1}{M}\!\sum_{m=1}^{M}
    -\log (\frac{e^{\bm{w}^{\textnormal{T}}_{y_{i}}\bm{a}_{i}^{m}+ b_{y_{i}}}}
    {\sum_{j=1}^{C}e^{\bm{w}^{\textnormal{T}}_{j}\bm{a}_{i}^{m} + b_{j}}}),
\end{equation}
where $\bm{W} = [\bm{w}_{1},\dots, \bm{w}_{C}]^{\textnormal{T}} \in \mathcal{R}^{C \times A}$ and $\bm{b} = [b_{1},\dots, b_{C}]^{\textnormal{T}} \in \mathcal{R}^C$ are the weight matrix and biases corresponding to the final fully connected layer, respectively. 

Obviously, the naive implementation is computationally inefficient when $M$ is large, as the feature set is enlarged by $M$ times. In the following, we consider the case that $M$ grows to infinity, and find that an easy-to-compute upper bound can be derived for the loss function, leading to a highly efficient implementation.



In the case $M\rightarrow\infty$, we are in fact considering the expectation of the CE loss under all possible augmented features. Specifically, $\mathcal{L}_{\infty}$ is given by:
\begin{equation}
    \label{expectation}
    \mathcal{L}_{\infty}(\bm{W}, \bm{b}, \bm{\Theta}|\bm{\Sigma})
\!=\!\frac{1}{N}\!\sum_{i=1}^{N}\mathrm{E}_{\tilde{\bm{a}}_{i}}[
    -\log(
        \frac{e^{\bm{w}^{\textnormal{T}}_{y_{i}}\tilde{\bm{a}}_{i}+ b_{y_{i}}}}
    {\sum_{j=1}^{C}e^{\bm{w}^{\textnormal{T}}_{j}\tilde{\bm{a}}_{i} + b_{j}}}
    )].
\end{equation}
If $\mathcal{L}_{\infty}$ can be computed efficiently, then we can directly minimize it without explicitly sampling augmented features. However, Eq. (\ref{expectation}) is difficult to compute in its exact form. Alternatively, we find that it is possible to derive an easy-to-compute upper bound for $\mathcal{L}_{\infty}$, as given by the following proposition.

\begin{proposition}
        \label{proposition}
    Suppose that $\tilde{\bm{a}}_{i} \sim \mathcal{N}(\bm{a}_{i}, \lambda\Sigma_{y_i})$. Then we have an upper bound of $\mathcal{L}_{\infty}$, given by:
    \begin{equation}
        \label{proposition_1}
        \begin{split}
            {\mathcal{L}}_{\infty}
         \leq  \frac{1}{N} \! \sum_{i=1}^{N} \! - \log(\!
              \frac{e^{
             \bm{w}^{\textnormal{T}}_{y_{i}}\bm{a}_{i} + b_{y_{i}}
         }}{\sum_{j=1}^{C}\!e^{
            \bm{w}^{\textnormal{T}}_{j}\bm{a}_{i}+b_{j}+\frac{\lambda}{2}\bm{v}^{\textnormal{T}}_{jy_{i}}\!\Sigma_{y_i}\!\bm{v}_{jy_{i}}
         }}\!) \triangleq \overline{\mathcal{L}}_{\infty},
        \end{split}
    \end{equation}
    where $\bm{v}_{jy_{i}} = \bm{w}_{j} - \bm{w}_{y_{i}}$.
    \end{proposition}
    \begin{proof}
        According to the definition of $\mathcal{L}_{\infty}$ in Eq. (\ref{expectation}), we have:
    \begin{align}
        \label{Js_1}
        \mathcal{L}_{\infty} 
        = &  \frac{1}{N} \sum_{i=1}^{N}\mathrm{E}_{\tilde{\bm{a}}_{i}}[
            \log(
            \sum_{j=1}^{C}
        e^{\bm{v}^{\textnormal{T}}_{jy_{i}}\tilde{\bm{a}}_{i} + (b_{j} - b_{y_{i}})}
            )] \\
            \label{Js_3}
        \leq &  \frac{1}{N}
        \sum_{i=1}^{N} 
        \log(\sum_{j=1}^{C}\mathrm{E}_{\tilde{\bm{a}}_{i}}
        [e^{\bm{v}^{\textnormal{T}}_{jy_{i}}\tilde{\bm{a}}_{i} + (b_{j} - b_{y_{i}})}]) \\
        \label{Js_4}
          = & \frac{1}{N}
        \sum_{i=1}^{N}
        \log(\sum_{j=1}^{C}e^{\bm{v}^{\textnormal{T}}_{jy_{i}}\bm{a}_{i}\!+\!(b_{j}\!-\!b_{y_{i}})
        \!+\!\frac{\lambda}{2}\bm{v}^{\textnormal{T}}_{jy_{i}}\Sigma_{y_i}\bm{v}_{jy_{i}}}
        )  
        \\
        \label{Js_5}
        = & \overline{\mathcal{L}}_{\infty}.
    \end{align}
    In the above, the Inequality (\ref{Js_3}) follows from the Jensen's inequality $\mathrm{E}[\log \!X] \leq \log\mathrm{E}[X]$, as the logarithmic function $\log(\cdot)$ is concave. The Eq. (\ref{Js_4}) is obtained by leveraging the moment-generating function:
    \begin{equation*}
        \mathrm{E}[e^{tX}] = e^{t\mu + \frac{1}{2} \sigma^2 t^2},\ \ X \sim \mathcal{N}(\mu,\sigma^2),
    \end{equation*}
    due to the fact that $\bm{v}^{\textnormal{T}}_{jy_{i}}\tilde{\bm{a}}_{i}\!+\!(b_{j}\!-\!b_{y_{i}})$ is a Gaussian random variable, i.e.,
    \begin{equation*}
        \label{proof_1}
        \begin{split}
            \bm{v}^{\textnormal{T}}_{jy_{i}}\tilde{\bm{a}}_{i}\!+\!(b_{j}\!-\!b_{y_{i}}) \sim 
             \mathcal{N}(\bm{v}^{\textnormal{T}}_{jy_{i}}\bm{a}_{i}\!+\!(b_{j}\!-\!b_{y_{i}}), 
            \lambda\bm{v}^{\textnormal{T}}_{jy_{i}}\Sigma_{y_i}\bm{v}_{jy_{i}}).
        \end{split}
         \qedhere
    \end{equation*}
    
    \end{proof}

Essentially, Proposition \ref{proposition} provides a surrogate loss for our implicit data augmentation algorithm. Instead of minimizing the exact loss function $\mathcal{L}_{\infty}$, we can optimize its upper bound $\overline{\mathcal{L}}_{\infty}$ in a much more efficient way. Therefore, the proposed ISDA boils down to a novel robust loss function, which can be easily adopted by most deep models.
In addition, we can observe that when $\lambda \rightarrow 0$, which means no features are augmented,
$\overline{\mathcal{L}}_{\infty}$ reduces to the standard CE loss.

In summary, the proposed ISDA approach can be simply plugged into deep networks as a robust loss function, and efficiently optimized with the stochastic gradient descent (SGD) algorithm. We present the pseudo code of ISDA in Algorithm \ref{alg}. 
\begin{center}
    \vskip -0.1in
        \begin{algorithm}[H]
            \caption{The ISDA algorithm.}
            \label{alg}
        \begin{algorithmic}[1]
            \STATE {\bfseries Input:} $\mathcal{D}$, $\lambda_0$
            \STATE Randomly initialize
            $\bm{W}, \bm{b}$ and $\bm{\Theta}$ 
            \FOR{$t=0$ {\bfseries to} $T$}
            \STATE Sample a mini-batch $\{ \bm{x}_i, {y_i} \}_{i=1}^B$  from $\mathcal{D}$
            \STATE Compute $\bm{a}_{i} = G(\bm{x}_{i}, \bm{\Theta})$
            \STATE Estimate the covariance matrices $\Sigma_{1}$, $\Sigma_{2}$, $...$, $\Sigma_{C}$
            \STATE Compute $\overline{\mathcal{L}}_{\infty}$
            according to Eq. (\ref{proposition_1})
            \STATE Update $\bm{W}, \bm{b}$, $\bm{\Theta}$ with SGD
            \ENDFOR
            \STATE {\bfseries Output:} $\bm{W}, \bm{b}$ and $\bm{\Theta}$
            
        \end{algorithmic}
        \end{algorithm}
\end{center}

\subsection{Complexity of ISDA}
\label{Complexity}
Here we present a theoretical analysis to show that ISDA does not involve notable additional computational cost. As shown above, ISDA requires extra computation for estimating the covariance matrices and computing the upper bound of the excepted loss. For a single sample, the computational complexity of the former is $O(D^2)$ (using the online update formulas Eqs. (\ref{ave})(\ref{cv})), while that of the later is $O(C \! \times \! D^2)$, where $D$ is the dimension of feature space. In comparison, a typical ConvNet with $L$ layers requires $O(D^2 \!\times\! K^2  \!\times\! H  \!\times\! W \! \times \! L)$ operations, where $K$ is the filter kernel size, and $H$ and $W$ are the height and width of feature maps. Consider ResNet-110 on CIFAR (C10 \& C100) as an example, for which we have $K\!=\!3$, $H\!=\!W\!=\!8$ and $L\!=\!109$ (ignoring the last FC-layer), then the extra computation cost of ISDA is up to \emph{three orders of magnitude less} than the total computation cost of the network. In our experiments, the results of both theoretical and practical cost of ISDA are provided in Table \ref{ComputationalCost}.


\section{ISDA for Deep Semi-supervised Learning}
\label{semi_ISDA}
Deep networks have achieved remarkable success in supervised learning tasks when fueled by sufficient annotated training data. 
However, obtaining abundant annotations is usually costly and time-consuming in practice. In comparison, collecting training samples without labels is a relatively easier task. The goal of deep semi-supervised learning is to improve the performance of deep networks by leveraging both labeled and unlabeled data simultaneously \cite{kingma2014semi, rasmus2015semi, laine2016temporal, tarvainen2017mean}. In this section, we further introduce how to apply the proposed algorithm to semi-supervised learning tasks. 


It is not straightforward to directly implement the aforementioned ISDA algorithm in semi-supervised learning, since unlabeled samples do not have ground truth labels, which are essential to compute the supervised ISDA loss $\overline{\mathcal{L}}_{\infty}$. Inspired by other consistency based semi-supervised learning methods \cite{tarvainen2017mean, laine2016temporal, miyato2018virtual, luo2018smooth}, we propose a semantic consistency training approach to exploit unlabeled data in ISDA. Our major insight here is that the prediction of a given sample should not be significantly changed when it is augmented, because ISDA performs class identity preserving semantic transformations. In specific, we first augment the deep features of unlabeled data, and then minimize the KL-divergence between the predictions of the augmented samples and the corresponding original samples. Interestingly, we find it feasible to derive a closed-form upper-bound of the expected KL-divergence as a surrogate loss, which makes our semi-supervised ISDA algorithm highly efficient,  similar to the case of supervised learning. ISDA can be incorporated into state-of-the-art deep semi-supervised learning algorithms to further improve their performance.

Consider training a deep network with weights $\bm{\Theta}$ on a labeled training set $\mathcal{D}^{\text{L}} = \{(\bm{x}_{i}^{\text{L}}, y_{i}^{\text{L}})\}$ and an unlabeled training set $\mathcal{D}^{\text{U}} = \{\bm{x}_{i}^{\text{U}}\}$. For labeled samples, we simply minimize the upper bound in Proposition \ref{proposition}. For unlabeled samples, given an input $\bm{x}_{i}^{\text{U}}$, we first obtain its deep feature $\bm{a}_{i}^{\text{U}}$ and the corresponding prediction $\bm{p}_i^{\text{U}} \in (0, 1)^C$. Then we obtain the augmented feature: 
\begin{equation}
    \tilde{\bm{a}}_{i}^{\text{U}} \sim \mathcal{N}(\bm{a}_{i}^{\text{U}}, \lambda\Sigma_{\tilde{y}_i^{\text{U}}}),\ \  \tilde{y}_i^{\text{U}} = \arg\max_{j} p_{ij}^{\text{U}},
\end{equation}
where $p_{ij}^{\text{U}}$ indicates the $j^{\textnormal{th}}$ element of $\bm{p}_i^{\text{U}}$ and $\tilde{y}_i^{\text{U}}$ is the pseudo label of $\bm{x}_{i}^{\text{U}}$. The covariance matrix $\Sigma_{\tilde{y}_i^{\text{U}}}$ is estimated using the deep features of labeled data.

Then the prediction of $\tilde{\bm{a}}_{i}^{\text{U}}$ is able to be calculated, denoted as $\tilde{\bm{p}}_i^{\text{U}}$. Since ISDA performs transformations that do not affect the class identity of samples, we enforce $\bm{p}_i^{\text{U}}$ and $\tilde{\bm{p}}_i^{\text{U}}$ to be similar by minimizing the KL-divergence between them. As $\tilde{\bm{a}}_{i}^{\text{U}}$ is a random variable, a straightforward approach to achieve that is to obtain $M$ samples of $\tilde{\bm{a}}_{i}^{\text{U}}$ and minimize the averaged KL-divergence of all the $M$ samples. However, as discussed in Section \ref{sec_4_2}, such a naive implementation is inefficient due to the enlarged feature set. To alleviate the problem, we consider the case where $M \to \infty$ and minimize the expected KL-divergence over $\tilde{\bm{a}}_{i}^{\text{U}}$:
\begin{equation}
    \min_{\bm{W}, \bm{b}, \bm{\Theta}}\mathrm{E}_{\tilde{\bm{a}}_{i}^{\text{U}}} [\textnormal{D}_{\text{KL}}(\bm{p}_i^{\text{U}}|\!|\tilde{\bm{p}}_i^{\text{U}})].
\end{equation}
Here, we treat $\bm{p}_i^{\text{U}}$ as a constant to stabilize the training procedure following \cite{miyato2018virtual, laine2016temporal, tarvainen2017mean}. Formally, a semantic consistency loss for unlabeled data is given by:
\begin{equation}
    \label{ul_loss}
    \begin{split}
        \mathcal{L}^{\text{U}}_{\infty}(& \bm{W}, \bm{b}, \bm{\Theta}|\bm{\Sigma})\!=\! \\
    & \frac{1}{N}\!\sum_{i=1}^{N}\!\sum_{k=1}^{C} p_{ik}^{\text{U}}
    \mathrm{E}_{\tilde{\bm{a}}_{i}}[
        -\log(
            \frac{e^{\bm{w}^{\textnormal{T}}_{k}\tilde{\bm{a}}_{i}+ b_{k}}}
        {\sum_{j=1}^{C}e^{\bm{w}^{\textnormal{T}}_{j}\tilde{\bm{a}}_{i} + b_{j}}}
        )].
    \end{split}
\end{equation}
It is difficult to compute Eq. (\ref{ul_loss}) in the exact form. Therefore, instead of directly using Eq. (\ref{ul_loss}) as the loss function, we show in the following proposition that a closed-form upper bound of $\mathcal{L}^{\text{U}}_{\infty}$ can be obtained as a surrogate loss. Similar to the supervised learning case, our semi-supervised ISDA algorithm amounts to minimizing a novel robust loss, and can be implemented efficiently.

\begin{proposition}
        \label{proposition_ul}
    Suppose that $\tilde{\bm{a}}_{i}^{\textnormal{U}} \sim \mathcal{N}(\bm{a}_{i}^{\textnormal{U}}, \lambda\Sigma_{\tilde{y}_i^{\textnormal{U}}})$. Then we have an upper bound of $\mathcal{L}^{\textnormal{U}}_{\infty}$, given by:
    \begin{equation}
        \label{proposition_1_ul}
        \begin{split}
            &\mathcal{L}^{\textnormal{U}}_{\infty}
         \leq \! \frac{1}{N}\!\!\sum_{i=1}^{N}\!\sum_{k=1}^{C}\!\!- p_{ik}^{\textnormal{U}}\log(\!
              \frac{e^{
             \bm{w}^{\textnormal{T}}_{k}\bm{a}_{i} + b_{k}
         }}{\sum_{j=1}^{C} \! e^{
            \bm{w}^{\textnormal{T}}_{j}\!\bm{a}_{i}\!+\!b_{j}\!+\!\frac{\lambda}{2}\!\bm{v}^{\textnormal{T}}_{jk} \! (\Sigma_{\tilde{y}_i^{\textnormal{U}}}) \bm{v}_{jk}
         }}\!) \triangleq \overline{\mathcal{L}}^{\textnormal{U}}_{\infty},
        \end{split}
    \end{equation}
    where $\bm{v}_{jk} = \bm{w}_{j} - \bm{w}_{k}$.
\end{proposition}
\begin{proof}
    According to Eq. (\ref{ul_loss}), we have:
\begin{align}
    \label{ul_1}
    \mathcal{L}^{\textnormal{U}}_{\infty}
    = & 
    \sum_{k=1}^{C} \left\{\frac{1}{N}\!\sum_{i=1}^{N}
    p_{ik}^{\textnormal{U}}
    \mathrm{E}_{\tilde{\bm{a}}_{i}}[
        -\log(
            \frac{e^{\bm{w}^{\textnormal{T}}_{k}\tilde{\bm{a}}_{i}+ b_{k}}}
        {\sum_{j=1}^{C}e^{\bm{w}^{\textnormal{T}}_{j}\tilde{\bm{a}}_{i} + b_{j}}})]\right\}
    \\
    \label{ul_2}
    \begin{split}
    \leq & 
    \!\sum_{k=1}^{C}\! 
    \left[\!\frac{1}{N}\!\! \sum_{i=1}^{N} \!-p_{ik}^{\textnormal{U}} \!\log(\!
              \frac{e^{
             \bm{w}^{\textnormal{T}}_{k}\bm{a}_{i} + b_{k}
         }}{\sum_{j=1}^{C} \! e^{
            \bm{w}^{\textnormal{T}}_{j}\!\bm{a}_{i}\!+\!b_{j}\!+\!\frac{\lambda}{2}\!\bm{v}^{\textnormal{T}}_{jk} \! (\Sigma_{\tilde{y}_i^{\textnormal{U}}}) \bm{v}_{jk}
         }}\!)\!\right]
    \end{split}
    \\
    \label{ul_3}
    = & \overline{\mathcal{L}}^{\textnormal{U}}_{\infty}.
\end{align}
In the above, Inequality (\ref{ul_2}) follows from the conclusion of Proposition \ref{proposition}.
\qedhere
\end{proof} 

In sum, the loss function of our method is given by:
\begin{equation}
    \label{overall}
    \overline{\mathcal{L}}^{\textnormal{L}}_{\infty} + \eta_1 \overline{\mathcal{L}}^{\textnormal{U}}_{\infty} + \eta_2 \mathcal{L}_{\textnormal{regularization}},
\end{equation}
where $\overline{\mathcal{L}}^{\textnormal{L}}_{\infty}$ is the ISDA loss on labeled data. As most deep semi-supervised learning algorithms model unlabeled data in regularization terms, they can be conveniently integrated with ISDA by appending the corresponding regularization term $\mathcal{L}_{\textnormal{regularization}}$ to the loss function. The coefficient $\eta_1$ and $\eta_2$ are pre-defined hyper-parameters to determine the importance of different regularization terms.

%% file: experiments.tex
\section{Experiments}


In this section, we empirically validate the proposed algorithm on several tasks.
First, we present the experimental results of supervised image classification on widely used benchmarks, i.e., CIFAR \cite{krizhevsky2009learning} and ImageNet \cite{5206848}. Second, we show the performance of several deep semi-supervised learning algorithms with and without ISDA on CIFAR \cite{krizhevsky2009learning} and SVHN \cite{goodfellow2013multi}. Third, we apply ISDA to the semantic segmentation task on the Cityscapes dataset \cite{cordts2016cityscapes}. 

In addition, to demonstrate that ISDA encourages models to learn better representations, we conduct experiments by employing the models trained with ISDA as backbones for the object detection task and the instance segmentation task on the MS COCO dataset \cite{lin2014microsoft}, which are presented in Appendix \ref{COCO_results}.



Furthermore, a series of analytical experiments are conducted to provide additional insights into our algorithm. We provide the visualization results of both the augmented samples and the representations learned by deep networks. We also present the empirical results to check the tightness of the upper bound used by ISDA. The performance of explicit and implicit semantic data augmentation is compared. Finally, ablation studies and sensitivity tests are conducted to show how the components and hyper-parameters affect the performance of ISDA.

\subsection{Experimental Setups for Image Classification}
\label{sec:dataset-baseline}
\textbf{Datasets.} We use three image classification benchmarks in the experiments. (1) The two \emph{CIFAR} datasets consist of 32x32 colored natural images in 10 classes for CIFAR-10 and 100 classes for CIFAR-100, with 50,000 images for training and 10,000 images for testing, respectively. 
(2) \emph{The Street View House Numbers (SVHN)} dataset \cite{goodfellow2013multi} consists of 32x32 colored images of digits. 73,257 images for training, 26,032 images for testing and 531,131 images for additional training are provided. 
(3) \emph{ImageNet} is a 1,000-class dataset from ILSVRC2012\cite{5206848}, providing 1.2 million images for training and 50,000 images for validation. 

\textbf{Validation set and data pre-procession.}
(1) On CIFAR, in our supervised learning experiments, we hold out 5,000 images from the training set as the validation set to search for the hyper-parameter $\lambda_0$. These samples are also used for training after an optimal $\lambda_0$ is selected, and the results on the test set are reported. Images are normalized with channel means and standard deviations for pre-processing. 
We follow the basic data augmentation operations in \cite{He_2016_CVPR, 2016arXiv160806993H, wang2020collaborative, wang2021revisiting}: 4 pixels are padded at each side of the image, followed by a random 32x32 cropping combined with random horizontal flipping. In semi-supervised learning experiments, we hold out $25\%$ of labeled images as the validation set to select $\lambda_0$, $\eta_1$ and $\eta_2$, following \cite{miyato2018virtual}. Similarly, these samples are also used for training with the selected hyper-parameters. Following the common practice of semi-supervised learning \cite{miyato2018virtual, tarvainen2017mean, laine2016temporal, verma2019interpolation}, we apply ZCA whitening for pre-processing and random 2x2 translation followed by random horizontal flip for basic augmentation.
(2) The SVHN dataset is used for semi-supervised learning experiments, where $25\%$ of labeled images are held out as the validation set. The validation set is put back for training after the hyper-parameter searching. Following \cite{luo2018smooth, tarvainen2017mean}, we perform random 2x2 translation to augment the training data.
(3) On ImageNet, we adopt the same augmentation configurations as \cite{krizhevsky2012imagenet,He_2016_CVPR,2016arXiv160806993H, yang2020resolution, wang2020glance}.

\begin{table*}[t]
    \centering

    \caption{Single crop error rates (\%) of different deep networks on the validation set of ImageNet. We report the results of our implementation with and without ISDA. The better results are \textbf{bold-faced}, while the numbers in brackets denote the performance improvements achieved by ISDA. For a fair comparison, we present the baselines reported by other papers \cite{2016arXiv160806993H, yun2019cutmix} as well. We also report the theoretical computational overhead and the additional training time introduced by ISDA in the last two columns, which is obtained with 8 Tesla V100 GPUs.
    }
    \vskip -0.13in
    \label{ImageNet Results}
    \setlength{\tabcolsep}{2mm}{
    \vspace{5pt}
    \renewcommand\arraystretch{1.21}
    \begin{tabular}{c|c||c|c|c|c|c}
    \hline
    \multirow{2}{*}{Networks}  &  \multirow{2}{*}{Params} &  \multicolumn{3}{c|}{Top-1/Top-5 Error Rates (\%)}  &  Additional Cost &   Additional Cost  \\
    \cline{3-5}
     &  &Reported in \cite{2016arXiv160806993H, yun2019cutmix} & Our Implementation & ISDA & (Theoretical) &  (Wall Time) \\
    \hline
    ResNet-50 \cite{He_2016_CVPR} & 25.6M & 23.7 / 7.1 & 23.0 / 6.8  & \textbf{21.9$_{(1.1)}$ / 6.3} & 0.25\% &7.6\%\\
    ResNet-101 \cite{He_2016_CVPR} & 44.6M  & 21.9 / 6.3 & 21.7 / 6.1  & \textbf{20.8$_{(0.9)}$ / 5.7}  & 0.13\% & 7.4\%\\
    ResNet-152 \cite{He_2016_CVPR} & 60.3M & 21.7 / 5.9 & 21.3 / 5.8 &  \textbf{20.3$_{(1.0)}$ / 5.5} & 0.09\% &5.4\%\\
    \hline
    DenseNet-BC-121 \cite{2016arXiv160806993H} & 8.0M & 25.0 / 7.7 & 23.7 / 6.8 & \textbf{23.2$_{(0.5)}$ / 6.6}  & 0.20\% &5.6\%\\
    DenseNet-BC-265 \cite{2016arXiv160806993H} & 33.3M & 22.2 / 6.1 & 21.9 / 6.1  &  \textbf{21.2$_{(0.7)}$ / 6.0} & 0.24\% &5.4\%\\
    \hline
    ResNeXt-50, 32x4d \cite{xie2017aggregated} & 25.0M & \ \ --\ \ / \ \ --\ \  & 22.5 / 6.4  & \textbf{21.3$_{(1.2)}$ / 5.9}   & 0.24\% &6.6\%\\
    ResNeXt-101, 32x8d \cite{xie2017aggregated} & 88.8M & \ \ --\ \  / \ \ --\ \  & 21.1 / 5.9  & \textbf{20.1$_{(1.0)}$ / 5.4} & 0.06\% &7.9\%\\
    \hline
    \end{tabular}}
    \vskip -0.1in
\end{table*}

\begin{table*}[t]
    \centering

    \caption{Evaluation of ISDA on CIFAR with different models. We report mean values and standard deviations in five independent experiments. The better results are \textbf{bold-faced}.}
    \vskip -0.125in
    \label{different_networks}
    \setlength{\tabcolsep}{4mm}{
    \vspace{5pt}
    \renewcommand\arraystretch{1.21}
    \begin{tabular}{c|c||cc|cc}
    \hline
    \multirow{2}{*}{Networks}  &  \multirow{2}{*}{Params}  & \multicolumn{2}{c|}{CIFAR-10}  &   \multicolumn{2}{c}{CIFAR-100}    \\
     &  &  Basic & ISDA  & Basic &  ISDA \\
    \hline
    ResNet-32 \cite{He_2016_CVPR} & 0.5M & 7.39 $\pm$ 0.10\% & \textbf{7.09 $\pm$ 0.12\%}  &  31.20 $\pm$ 0.41\% & \textbf{30.27 $\pm$ 0.34\%}\\
    ResNet-110 \cite{He_2016_CVPR} & 1.7M  & 6.76 $\pm$ 0.34\%& \textbf{6.33 $\pm$ 0.19\%} &  28.67 $\pm$ 0.44\% & \textbf{27.57 $\pm$ 0.46\%}\\
    SE-ResNet-110 \cite{hu2018squeeze} & 1.7M & 6.14 $\pm$ 0.17\%& \textbf{5.96 $\pm$ 0.21\%} &27.30 $\pm$ 0.03\%& \textbf{26.63 $\pm$ 0.21\%}\\
    Wide-ResNet-16-8 \cite{Zagoruyko2016WideRN} & 11.0M & 4.25 $\pm$ 0.18\%&\textbf{4.04 $\pm$ 0.29\%} & 20.24 $\pm$ 0.27\%& \textbf{19.91 $\pm$ 0.21\%}\\
    Wide-ResNet-28-10 \cite{Zagoruyko2016WideRN} & 36.5M & 3.82 $\pm$ 0.15\% &  \textbf{3.58 $\pm$ 0.15\%} &  18.53 $\pm$ 0.07\% & \textbf{17.98 $\pm$ 0.15\%}\\
    ResNeXt-29, 8x64d \cite{xie2017aggregated} & 34.4M & 3.86 $\pm$ 0.14\% &  \textbf{3.67 $\pm$ 0.12\%} &  18.16 $\pm$ 0.13\%&  \textbf{17.43 $\pm$ 0.25\%}\\
    DenseNet-BC-100-12 \cite{2016arXiv160806993H} & 0.8M & 4.90 $\pm$ 0.08\% &  \textbf{4.54 $\pm$ 0.07\%}&  22.61 $\pm$ 0.10\%& \textbf{22.10 $\pm$ 0.34\%}\\
    Shake-Shake (26, 2x32d) \cite{gastaldi2017shake} & 3.0M &  3.45 $\pm$ 0.01\% & \textbf{3.20 $\pm$ 0.03\%}  &  20.12 $\pm$ 0.39\% & \textbf{19.45 $\pm$ 0.16\%} \\
    Shake-Shake (26, 2x112d) \cite{gastaldi2017shake} & 36.4M &  2.92 $\pm$ 0.02\% & \textbf{2.61 $\pm$ 0.09\%}  & 17.42 $\pm$ 0.44\% & \textbf{16.73 $\pm$ 0.18\%}\\
    \hline
    \end{tabular}}
    \vskip -0.1in
\end{table*}

\begin{table*}[t]
    \centering
    \caption{The theoretical computational overhead and the empirical additional time consumption of ISDA on CIFAR. The results are obtained with a single Tesla V100 GPU.}
    \vskip -0.125in
    \label{ComputationalCost}
    \setlength{\tabcolsep}{1.5mm}{
    \renewcommand\arraystretch{1.21}
    \begin{tabular}{c|c|cc|c|cc|c}
    \hline
    \multirow{4}{*}{Networks} &  & \multicolumn{3}{c|}{CIFAR-10} & \multicolumn{3}{c}{CIFAR-100} \\
    \cline{3-8} & Standard & \multicolumn{2}{c|}{Additional Cost} & \multirowcell{3}{Additional Cost\\(Wall Time)} & \multicolumn{2}{c|}{Additional Cost} & \multirowcell{3}{Additional Cost\\(Wall Time)} \\

    & FLOPs & \multicolumn{2}{c|}{(Theoretical)} & & \multicolumn{2}{c|}{(Theoretical)} & \\

    \cline{3-4} \cline{6-7} & & Absolute & Relative & & Absolute & Relative & \\
    \hline 
 
    ResNet-32 \cite{He_2016_CVPR} & 69.43M & 0.05M & 0.07\% & 1.85\% & 0.44M & 0.63\%  & 1.85\%\\
    ResNet-110 \cite{He_2016_CVPR}& 254.20M & 0.05M & 0.02\%  & 4.48\%  & 0.44M & 0.17\% & 2.94\%\\
    SE-ResNet-110 \cite{hu2018squeeze}& 254.73M & 0.05M & 0.02\% & 2.17\% & 0.44M & 0.17\%  & 1.08\%\\
    Wide-ResNet-16-8 \cite{Zagoruyko2016WideRN}& 1.55G & 2.96M & 0.19\% & 3.39\%  & 27.19M & 1.76\%  & 10.77\%\\
    Wide-ResNet-28-10 \cite{Zagoruyko2016WideRN}& 5.25G & 4.61M & 0.09\% & 2.37\%  & 42.46M & 0.81\% & 12.35\%\\
    ResNeXt-29, 8x64d \cite{xie2017aggregated}& 5.39G & 11.80M & 0.22\% & 1.90\% & 109.00M & 2.02\% & 12.32\%\\
    DenseNet-BC-100-12 \cite{2016arXiv160806993H}& 292.38M & 1.35M & 0.46\% & 5.50\%  & 12.43M & 4.25\%  & 12.03\%\\
    Shake-Shake (26, 2x32d) \cite{gastaldi2017shake} & 426.69M & 0.19M & 0.04\% & 5.77\%  & 1.76M & 0.41\%  & 2.21\% \\
    Shake-Shake (26, 2x112d) \cite{gastaldi2017shake} & 5.13G & 2.31M & 0.05\% & 3.85\%  & 21.30M & 0.42\%  & 5.07\% \\
    \hline

    \end{tabular}}
    
    \vskip -0.1in
 \end{table*}

\begin{table*}[t]
    \centering
    \caption{Evaluation of ISDA with state-of-the-art \textit{non-semantic} augmentation techniques. `RA' and `AA' refer to RandAugment \cite{cubuk2020randaugment} and AutoAugment \cite{cubuk2018autoaugment}, respectively. We report mean values and standard deviations in five independent experiments. The better results are \textbf{bold-faced}.}
    \vskip -0.1in
    \label{complementary_result}
    \setlength{\tabcolsep}{1.4mm}{
    \renewcommand\arraystretch{1.21}
    \begin{tabular}{c|c|cc|cc|cc}
    \hline
    Dataset  &  Networks  &  Cutout \cite{devries2017improved}  &  Cutout + ISDA &   RA \cite{cubuk2020randaugment} &   RA + ISDA  &   AA \cite{cubuk2018autoaugment} &   AA + ISDA  \\
    \hline
    \multirow{3}{*}{CIFAR-10} & Wide-ResNet-28-10 \cite{Zagoruyko2016WideRN} & 2.99 $\pm$ 0.06\% & \textbf{2.83 $\pm$ 0.04\%} & 2.78 $\pm$ 0.03\% & \textbf{2.42 $\pm$ 0.13\%} &  2.65 $\pm$ 0.07\%& \textbf{2.56 $\pm$ 0.01\%}\\
    &Shake-Shake (26, 2x32d) \cite{gastaldi2017shake} & 3.16 $\pm$ 0.09\%& \textbf{2.93 $\pm$ 0.03\%}& 3.00 $\pm$ 0.05\% & \textbf{2.74 $\pm$ 0.03\%} & 2.89 $\pm$ 0.09\% & \textbf{2.68 $\pm$ 0.12\%} \\
    &Shake-Shake (26, 2x112d) \cite{gastaldi2017shake} & 2.36\%& \textbf{2.25\%}&  2.10\%& \textbf{1.76\%}  &  2.01\%& \textbf{1.82\%}\\
    \hline

    \multirow{3}{*}{CIFAR-100} & Wide-ResNet-28-10 \cite{Zagoruyko2016WideRN} & 18.05 $\pm$ 0.25\% & \textbf{16.95 $\pm$ 0.11\%} &17.30 $\pm$ 0.08\%& \textbf{15.97 $\pm$ 0.28\%}& 16.60 $\pm$ 0.40\%& \textbf{15.62 $\pm$ 0.32\%}\\
    &Shake-Shake (26, 2x32d) \cite{gastaldi2017shake} & 18.92  $\pm$ 0.21\%& \textbf{18.17 $\pm$ 0.08\%}&18.11 $\pm$ 0.24\% &\textbf{17.84 $\pm$ 0.16\%}&17.50 $\pm$ 0.19\% &\textbf{17.21 $\pm$ 0.33\%} \\
    &Shake-Shake (26, 2x112d) \cite{gastaldi2017shake} & 17.34  $\pm$ 0.28\%& \textbf{16.24 $\pm$ 0.20\%}&15.95 $\pm$ 0.15\% &\textbf{14.24 $\pm$ 0.07\%}&15.21 $\pm$ 0.20\% &\textbf{13.87 $\pm$ 0.26\%} \\
    \hline

    \end{tabular}}
    \vskip -0.1in
\end{table*}

\begin{table*}[t]
    \centering
    \caption{Comparisons with the state-of-the-art methods. We report mean values and standard deviations of the test errors in five independent experiments. The best results are \textbf{bold-faced}.}
    \vskip -0.1in
    \label{Tab02}
    \setlength{\tabcolsep}{6mm}{
    \renewcommand\arraystretch{1.21} 
    \begin{tabular}{l|cc|cc}
        \hline
    \multirow{2}{*}{Method} & \multicolumn{2}{c|}{ResNet-110} & \multicolumn{2}{c}{
        Wide-ResNet-28-10}\\
    &  CIFAR-10  &  CIFAR-100  &  CIFAR-10  &  CIFAR-100\\
    \hline
    Large Margin \cite{liu2016large} & 6.46 $\pm$ 0.20\% & 28.00 $\pm$ 0.09\% & 3.69 $\pm$ 0.10\% & 18.48 $\pm$ 0.05\%\\
    Disturb Label \cite{Xie2016DisturbLabelRC} & 6.61 $\pm$ 0.04\% & 28.46 $\pm$ 0.32\% & 3.91 $\pm$ 0.10\%& 18.56 $\pm$ 0.22\%\\
    Focal Loss \cite{Lin2017FocalLF} & 6.68 $\pm$ 0.22\% & 28.28 $\pm$ 0.32\% & 3.62 $\pm$ 0.07\% & 18.22 $\pm$ 0.08\%\\
    Center Loss \cite{wen2016discriminative} & 6.38 $\pm$ 0.20\% & 27.85 $\pm$ 0.10\% & 3.76 $\pm$ 0.05\% & {18.50 $\pm$ 0.25\%}\\
    L$_q$ Loss \cite{Zhang2018GeneralizedCE} & 6.69 $\pm$ 0.07\% & 28.78 $\pm$ 0.35\% & 3.78 $\pm$ 0.08\% & 18.43 $\pm$ 0.37\%\\
    Label Smoothing \cite{muller2019does} & 6.58 $\pm$ 0.42\% & 27.89 $\pm$ 0.20\% & 3.79 $\pm$ 0.16\% & 18.48 $\pm$ 0.24\%  \\
    DistributionNet \cite{yu2019robust}  & 6.35 $\pm$ 0.11\% & 28.18 $\pm$ 0.44\% & 3.68 $\pm$ 0.06\% & 18.34 $\pm$ 0.16\%\\
    \hline
    WGAN \cite{arjovsky2017wasserstein} & 6.63 $\pm$ 0.23\% & - & 3.81 $\pm$ 0.08\% & -\\
    CGAN \cite{mirza2014conditional} & 6.56 $\pm$ 0.14\% & 28.25 $\pm$ 0.36\% & 3.84 $\pm$ 0.07\% & 18.79 $\pm$ 0.08\%\\
    ACGAN \cite{odena2017conditional} & 6.32 $\pm$ 0.12\% & 28.48 $\pm$ 0.44\% & 3.81 $\pm$ 0.11\% & 18.54 $\pm$ 0.05\%\\
    infoGAN \cite{chen2016infogan} & 6.59 $\pm$ 0.12\% & 27.64 $\pm$ 0.14\% & 3.81 $\pm$ 0.05\% & 18.44 $\pm$ 0.10\%\\
    \hline
    Basic & 6.76 $\pm$ 0.34\% & 28.67 $\pm$ 0.44\% & - & -\\
    Basic + Dropout & {6.23 $\pm$ 0.11\%} & {27.11 $\pm$ 0.06\%} & 3.82 $\pm$ 0.15\% & 18.53 $\pm$ 0.07\%\\
    ISDA & 6.33 $\pm$ 0.19\% & 27.57 $\pm$ 0.46\% & - & -\\
    ISDA + Dropout & \textbf{5.98 $\pm$ 0.20\%} & \textbf{26.35 $\pm$ 0.30\%} & \textbf{3.58 $\pm$ 0.15\%} & \textbf{17.98 $\pm$ 0.15\%}\\
    \hline
    \end{tabular}}
    \vskip -0.1in
\end{table*}

\textbf{Non-semantic augmentation techniques. }
To study the complementary effects of ISDA to traditional data augmentation methods, three state-of-the-art non-semantic augmentation techniques are applied with and without ISDA. 
(1) \textit{Cutout} \cite{devries2017improved} randomly masks out square regions of inputs during training. 
(2) \textit{AutoAugment} \cite{cubuk2018autoaugment}  automatically searches for the best augmentation policy using reinforcement learning. 
(3) \textit{RandAugment} \cite{cubuk2020randaugment} searches for augmentation policies using grid search in a reduced searching space.

\textbf{Baselines for supervised learning.} 
Our method is compared with several baselines including state-of-the-art robust loss functions and generator-based semantic data augmentation methods. 
(1) \textit{Dropout} \cite{Srivastava2014DropoutAS} is a widely used regularization approach that randomly mutes some neurons during training.
(2) \textit{Large-margin softmax loss} \cite{liu2016large} introduces a large decision margin, measured by a cosine distance, to the standard CE loss.
(3) \textit{Disturb label} \cite{Xie2016DisturbLabelRC} is a regularization mechanism that randomly replaces a fraction of labels with incorrect ones in each iteration.
(4) \textit{Focal loss} \cite{Lin2017FocalLF} focuses on a sparse set of hard examples to prevent easy samples from dominating the training procedure.
(5) \textit{Center loss} \cite{wen2016discriminative} simultaneously learns a center of features for each class and minimizes the distances between the deep features and their corresponding class centers. 
(6) \textit{$L_q$ loss} \cite{Zhang2018GeneralizedCE} is a noise-robust loss function, using the negative Box-Cox transformation. 
(7) \textit{Label Smoothing} \cite{muller2019does} smooth the one-hot label to a soft one with equal values for other classes.
(8) \textit{DistributionNet} \cite{yu2019robust} models the deep features of training samples as Gaussian distributions, and learns the covariance automatically.
(9) For generator-based semantic augmentation methods, we train several state-of-the-art GANs \cite{arjovsky2017wasserstein, mirza2014conditional, odena2017conditional, chen2016infogan}, which are then used to generate extra training samples for data augmentation.

\textbf{Baselines for semi-supervised learning.} 
In semi-supervised learning experiments, the performance of ISDA is tested on the basis of several modern deep semi-supervised learning approaches.
(1) \textit{$\Pi$-model} \cite{laine2016temporal} enforces the model to have the same prediction on a sample with different augmentation and dropout modes.
(2) \textit{Temp-ensemble} \cite{laine2016temporal} attaches a soft pseudo label to each unlabeled sample by performing a moving average on the predictions of networks.
(3) \textit{Mean teacher} \cite{tarvainen2017mean} establishes a teacher network by performing an exponential moving average on the parameters of the model, and leverages the teacher network to produce supervision for unlabeled data.
(4) \textit{Virtual Adversarial Training (VAT)} \cite{miyato2018virtual} adds adversarial perturbation to each sample and enforces the model to have the same prediction on the perturbed samples and the original samples.

For a fair comparison, all methods are implemented with the same training configurations. Details for hyper-parameter settings are presented in Appendix \ref{hyper-para-baseline}.

\textbf{Implementation details.}
For supervised learning, we implement the ResNet, SE-ResNet, Wide-ResNet, ResNeXt, DenseNet and Shake-shake net on the two CIFAR datasets, and implement ResNet, DenseNet and ResNeXt on ImageNet. 
For semi-supervised learning, we implement the widely used CNN-13 network \cite{luo2018smooth, xie2019unsupervised, verma2019interpolation, laine2016temporal, tarvainen2017mean, miyato2018virtual, wang2020meta}. Details for implementing these models are given in Appendix \ref{Training_Details_sup} and Appendix \ref{Training_Details_semi_sup} for supervised learning and semi-supervised learning, respectively. 
The hyper-parameter $\lambda_0$ for ISDA is selected from the set $\{0.1, 0.25, 0.5, 0.75, 1\}$ according to the performance on the validation set. On ImageNet, due to GPU memory limitation, we approximate the covariance matrices by their diagonals, i.e., the variance of each dimension of the features. The best hyper-parameter $\lambda_0$ is selected from $\{1, 2.5, 5, 7.5, 10\}$. For semi-supervised learning tasks, the hyper-parameter $\eta_1$ is selected from $\{0.5, 1, 2\}$. In all experiments, the average test error of the last 10 epochs is calculated as the result to be reported.

\subsection{Supervised Image Classification}
\subsubsection{Main Results}
\textbf{Results on ImageNet. } 
Table \ref{ImageNet Results} presents the performance of ISDA on the large scale ImageNet dataset with state-of-the-art deep networks. It can be observed that ISDA significantly improves the generalization performance of these models. For example, the Top-1 error rate of ResNet-50 is reduced by $1.1\%$ via being trained with ISDA, approaching the performance of ResNet-101 ($21.9\%$ v.s. $21.7\%$) with $43\%$ fewer parameters. Similarly, the performance of ResNet-101+ISDA surpasses that of ResNet-152 with $26\%$ less parameters. Compared to ResNets, DenseNets generally suffer less from overfitting due to their architecture design, and thus appear to benefit less from our algorithm.


\textbf{Results on CIFAR. } 
We report the error rates of several modern deep networks with and without ISDA on CIFAR-10/100 in Table \ref{different_networks}. Similar observations to ImageNet can be obtained. On CIFAR-100, for relatively small models like ResNet-32 and ResNet-110, ISDA reduces test errors by about $1\%$, while for larger models like Wide-ResNet-28-10 and ResNeXt-29, 8x64d, our method outperforms the competitive baselines by nearly $0.7\%$.

\textbf{Complementing explicit augmentation techniques. } 
Table \ref{complementary_result} shows the experimental results with recently proposed traditional image augmentation methods (i.e., Cutout \cite{devries2017improved}, RandAugment \cite{cubuk2020randaugment} and AutoAugment \cite{cubuk2018autoaugment}). Interestingly, ISDA seems to be even more effective when these techniques exist. For example, when applying AutoAugment, ISDA achieves performance gains of $1.34\%$ and $0.98\%$ on CIFAR-100 with the Shake-Shake (26, 2x112d) and the Wide-ResNet-28-10, respectively. Note that these improvements are more significant than the standard situations. A plausible explanation for this phenomenon is that non-semantic augmentation methods help to learn a better feature representation, which makes semantic transformations in the deep feature space more reliable. The curves of test errors during training on CIFAR-100 with Wide-ResNet-28-10 are presented in Figure \ref{bound_fig}. It is clear that ISDA achieves a significant improvement after the third learning rate drop, and shows even better performance after the fourth drop.

\subsubsection{Comparisons with Other Approaches}
We compare ISDA with a number of competitive baselines described in Section \ref{sec:dataset-baseline}, ranging from robust loss functions to semantic data augmentation algorithms based on generative models.
The results are summarized in Table \ref{Tab02}.
One can observe that ISDA compares favorably with all these baseline algorithms.
On CIFAR-100, the best test errors of other robust loss functions are 27.85\% and 18.22\% with ResNet-110 and Wide-ResNet-28-10, respectively, while ISDA achieves 27.57\% and 17.98\%, respectively. Note that all the results with Wide-ResNet-28-10 use the dropout technique.


Among all GAN-based semantic augmentation methods, ACGAN gives the best performance, especially on CIFAR-10. However, these models generally suffer a performance reduction on CIFAR-100, which does not contain enough samples to learn a valid generator for each class. In contrast, ISDA shows consistent improvements on both the two datasets. In addition, GAN-based methods require additional computation to train the generators, and introduce significant overhead to the training process. In comparison, ISDA not only leads to lower generalization error, but is simpler and more efficient.

\begin{figure}[t]
    \centering
    \includegraphics[width=\columnwidth]{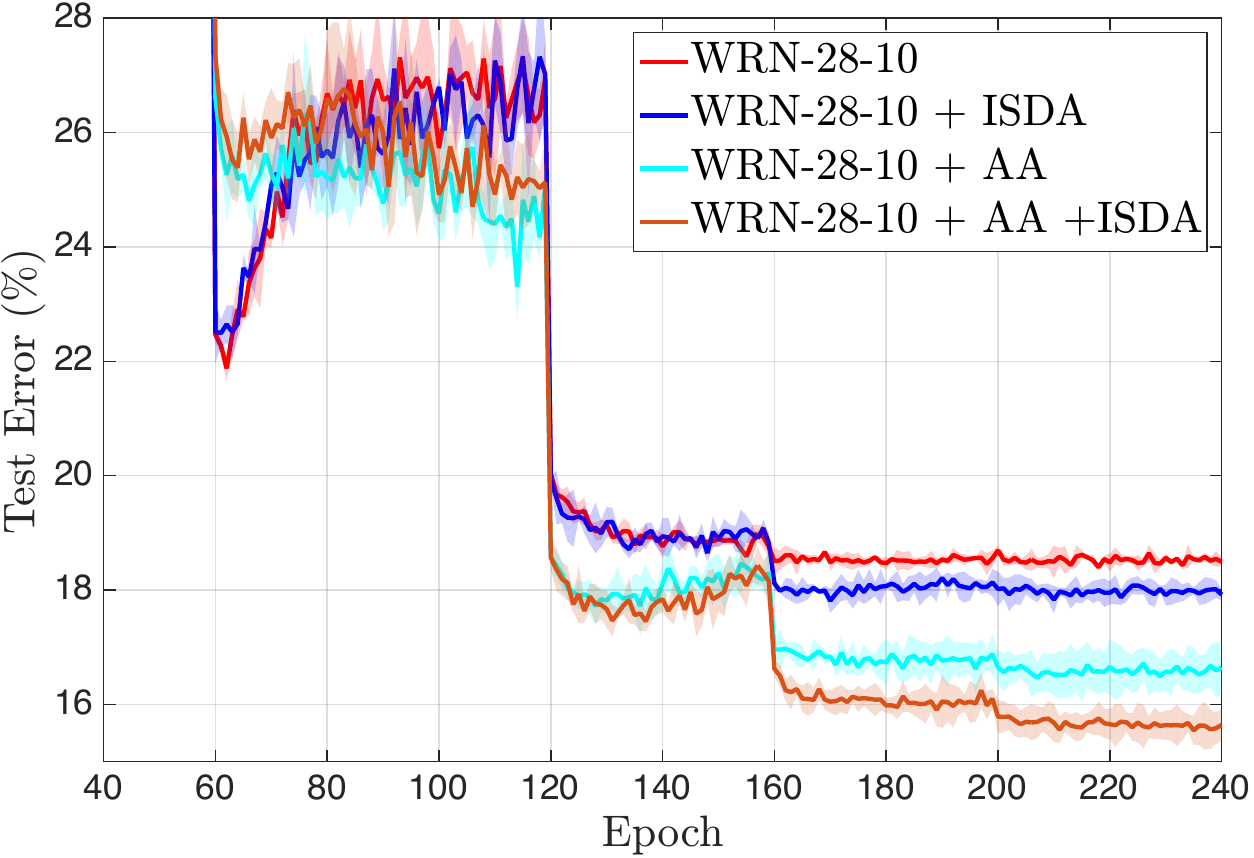}	
    \vskip -0.149in
    \caption{Curves of test errors on CIFAR-100 with Wide-ResNet (WRN). `AA' refers to  AutoAugment \cite{cubuk2018autoaugment}. \label{bound_fig}}  
    \vskip -0.1in
\end{figure}

\begin{table*}[t]
    \centering
    \caption{Performance of state-of-the-art semi-supervised learning algorithms with and without ISDA. We conduct experiments with different numbers of labeled samples. Mean results and standard deviations of five independent experiments are reported. The better results are \textbf{bold-faced}. }
    \vskip -0.109in
    \label{semi-supervised}
    \setlength{\tabcolsep}{1.5mm}{
    \renewcommand\arraystretch{1.21}
    \begin{tabular}{c|ccc|cc|c}
    \hline
    Dataset & \multicolumn{3}{c|}{CIFAR-10}  &  \multicolumn{2}{c|}{CIFAR-100}    &   SVHN \\
    \hline
    Labeled Samples & 1,000&2,000&4,000  &  10,000&20,000   &   500 \\
    \hline
    $\Pi$-model \cite{laine2016temporal} & 28.74  $\pm$ 0.48\% & 17.57  $\pm$ 0.44\%& 12.36  $\pm$ 0.17\% & 38.06  $\pm$ 0.37\%  & 30.80  $\pm$ 0.18\%   &  -  \\
    $\Pi$-model + ISDA  & \textbf{23.99 $\pm$ 1.30\%} & \textbf{14.90 $\pm$ 0.10\%}&  \textbf{11.35 $\pm$ 0.19\%} & \textbf{36.93 $\pm$ 0.28\%} & \textbf{ 30.03 $\pm$ 0.41\% }  &  -  \\
    \hline
    Temp-ensemble \cite{laine2016temporal}  & 25.15  $\pm$ 1.46\% &15.78  $\pm$ 0.44\% & 11.90  $\pm$ 0.25\%  & 41.56  $\pm$ 0.42\% & 35.35  $\pm$ 0.40\%  &  -  \\
    Temp-ensemble + ISDA & \textbf{22.77  $\pm$ 0.63\%} & \textbf{14.98  $\pm$ 0.73\%}&  \textbf{11.25  $\pm$ 0.31\%} & \textbf{40.47  $\pm$ 0.24\%} &  \textbf{34.58  $\pm$ 0.27\%} &  -  \\
    \hline
    Mean Teacher \cite{tarvainen2017mean} & 18.27  $\pm$ 0.53\%&13.45  $\pm$ 0.30\%&10.73  $\pm$ 0.14\%  & 36.03  $\pm$ 0.37\% &  30.00  $\pm$ 0.59\%&  4.18  $\pm$ 0.27\%  \\
    Mean Teacher + ISDA &\textbf{17.11 $\pm$ 1.03\%} & \textbf{12.35 $\pm$ 0.14\%}&\textbf{9.96 $\pm$ 0.33\%} & \textbf{34.60  $\pm$ 0.41\%} & \textbf{29.37  $\pm$ 0.30\%} &  \textbf{4.06 $\pm$ 0.11\%} \\
    \hline
    VAT \cite{miyato2018virtual} & 18.12   $\pm$ 0.82\%&13.93 $\pm$ 0.33\%&11.10   $\pm$ 0.24\%  &40.12  $\pm$ 0.12\%  &34.19 $\pm$ 0.69\%   &  5.10   $\pm$ 0.08\%  \\
    VAT + ISDA &\textbf{14.38 $\pm$ 0.18\%} & \textbf{11.52 $\pm$ 0.05\%}&\textbf{9.72 $\pm$ 0.14\%} &\textbf{36.04 $\pm$ 0.47\%}  &\textbf{30.97 $\pm$ 0.42\%}   &   \textbf{4.86 $\pm$ 0.18\%} \\

    \hline
    \end{tabular}}
    \vskip -0.1in
\end{table*}

\begin{table*}[t]
    \centering
    \caption{Performance of state-of-the-art semantic segmentation algorithms on Cityscapes with and without ISDA. `Multi-scale' and `Flip' denote employing the averaged prediction of multi-scale (\{0.75, 1, 1.25. 1.5\}) and left-right flipped inputs during inference. We present the results reported in the original papers in the `original' row. The numbers in brackets denote the performance improvements achieved by ISDA. The better results are \textbf{bold-faced}.
    }
    \vskip -0.12in
    \label{Segmentation}
    \setlength{\tabcolsep}{4mm}{
    \vspace{5pt}
    \renewcommand\arraystretch{1.21}
    \begin{tabular}{c|c|c|c|c}
    \hline
    \multirow{2}{*}{Method}  &  \multirow{2}{*}{Backbone} &  \multicolumn{3}{c}{mIoU (\%)} \\
    \cline{3-5}
    && Single Scale & Multi-scale & Multi-scale + Flip \\
    \hline
    PSPNet \cite{zhao2017pyramid} &  ResNet-101 &  77.46  & 78.10  & 78.41  \\
    PSPNet + ISDA &  ResNet-101 &  \textbf{78.72 ${(\uparrow1.26)}$} & \textbf{79.64 ${(\uparrow1.54)}$} & \textbf{79.44 ${(\uparrow1.03)}$}    \\
    \hline
    DeepLab-v3 (Original) \cite{chen2017rethinking}  &  ResNet-101 &  77.82 & 79.06 & 79.30    \\
    DeepLab-v3  &  ResNet-101 &  78.38  & 79.20  & 79.47  \\
    DeepLab-v3 + ISDA &  ResNet-101 &  \textbf{79.41 ${(\uparrow1.03)}$} & \textbf{80.30 ${(\uparrow1.10)}$} & \textbf{80.36 ${(\uparrow0.89)}$}    \\
    \hline
    \end{tabular}}
\end{table*}

\begin{figure*}[t]
    \begin{center}
    \vskip -0.1in
    \includegraphics[width=1.6\columnwidth]{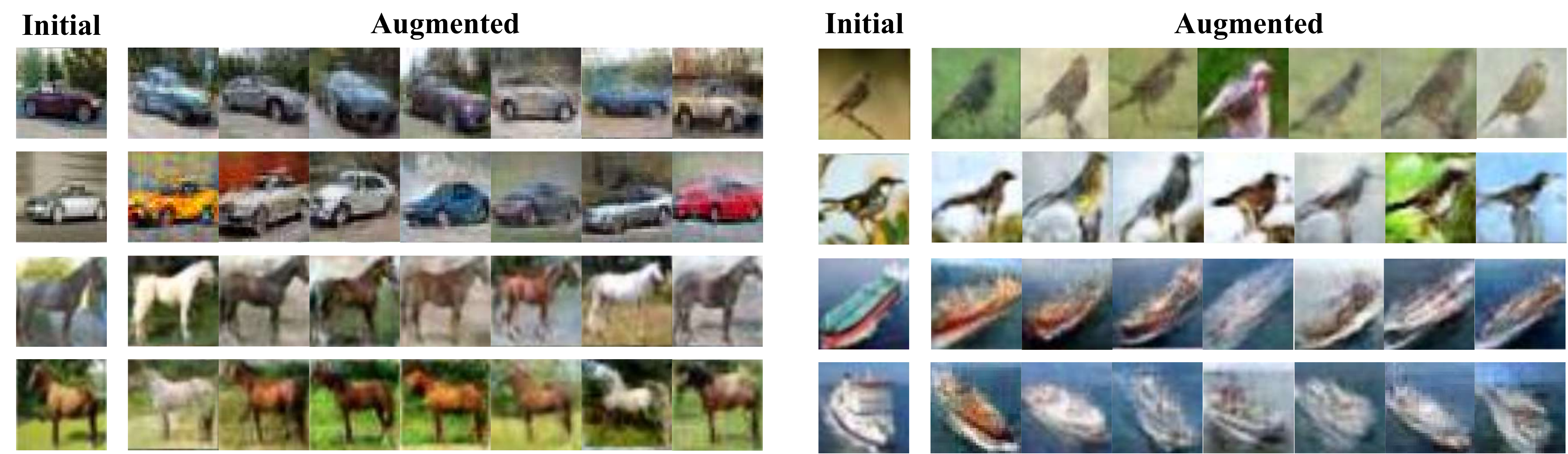}
    \vskip -0.1in
    \caption{Visualization of the semantically augmented images on CIFAR.}
    \vskip -0.1in
    \label{Visual}
    \end{center}
    \vskip -0.1in
  \end{figure*}

\subsubsection{Efficiency of ISDA}
\label{Computational_Cost_ISDA}
We report the theoretical computational overhead (measured by FLOPs) and the practical additional time consumption of ISDA in Table \ref{ImageNet Results} and Table \ref{ComputationalCost}. The results are obtained with 8 and 1 Tesla V100 GPUs on ImageNet and CIFAR using the public code we provide (\color{blue}{\textit{https://github.com/blackfeather-wang/ISDA-for-Deep-Networks}}\color{black}). One can see that ISDA increases the computational cost by no more than $1\%$ for most of the widely used deep networks on both ImageNet and CIFAR. The observation is consistent with our analysis in Section \ref{Complexity}. Notably, we find that our method involves a bit more computational consumption with ResNeXt and DenseNet on CIFAR-100. That is because the dimensions of the deep features learned by the two networks are relatively large. As a consequence, the size of the estimated covariance matrices is large, and thus more computation is required to update the covariance. Whereas, the additional computational cost is up to $4.25\%$, which will not significantly affect the training efficiency. Empirically, due to implementation issues, we observe $5\%$ to $7\%$ and $1\%$ to $12\%$ increase in training time on ImageNet and CIFAR, respectively.

\subsection{Semi-supervised Image Classification}
\label{semi_ISDA_result}
To test the performance of ISDA on semi-supervised learning tasks, we divide the training set into two parts, a labeled set and an unlabeled set, by randomly removing parts of the labels, and implement the proposed semi-supervised ISDA algorithm on the basis of several state-of-the-art semi-supervised learning algorithms. Results with different numbers of labeled samples are presented in Table \ref{semi-supervised}. It can be observed that ISDA complements these methods and further improves the generalization performance significantly. On CIFAR-10 with 4,000 labeled samples, adopting ISDA reduces the test error by $1.38\%$ with the VAT algorithm. In addition, ISDA performs even more effectively with fewer labeled samples and more classes. For example, VAT + ISDA outperforms the baseline by $3.74\%$ and $4.08\%$ on CIFAR-10 with 1,000 labeled samples and CIFAR-100 with 10,000 labeled samples, respectively.

\subsection{Semantic Segmentation on Cityscapes}
\label{Semantic_segmentation}
As ISDA augments training samples in the deep feature space, it can also be adopted for other classification based vision tasks, as long as the softmax cross-entropy loss is used. To demonstrate that, we apply the proposed algorithm to the semantic segmentation task on the Cityscapes dataset \cite{cordts2016cityscapes}, which contains 5,000 1024x2048 pixel-level finely annotated images and 20,000 coarsely annotated images from 50 different cities. Each pixel of the image is categorized among 19 classes. Following \cite{huang2019ccnet, liu2020structured}, we conduct our experiments on the finely annotated dataset and split it by 2,975/500/1,525 for training, validation and testing.

We first reproduce two modern semantic segmentation algorithms, PSPNet \cite{zhao2017pyramid} and Deeplab-v3 \cite{chen2017rethinking}, using the standard hyper-parameters\footnote{\color{blue}https://github.com/speedinghzl/pytorch-segmentation-toolbox\color{black}}. Then we fix the training setups and utilize ISDA to augment each pixel during training. Similar to ImageNet, we approximate the covariance matrices by their diagonals to save GPU memory. The results on the validation set are shown in Table \ref{Segmentation}. It can be observed that ISDA improves the performance of both the two baselines by nearly $1\%$ in terms of mIoU. In our experiments, we witness about $6\%$ increase in training time with ISDA.


  \begin{figure*}
    \begin{center}
    \includegraphics[width=1.8\columnwidth]{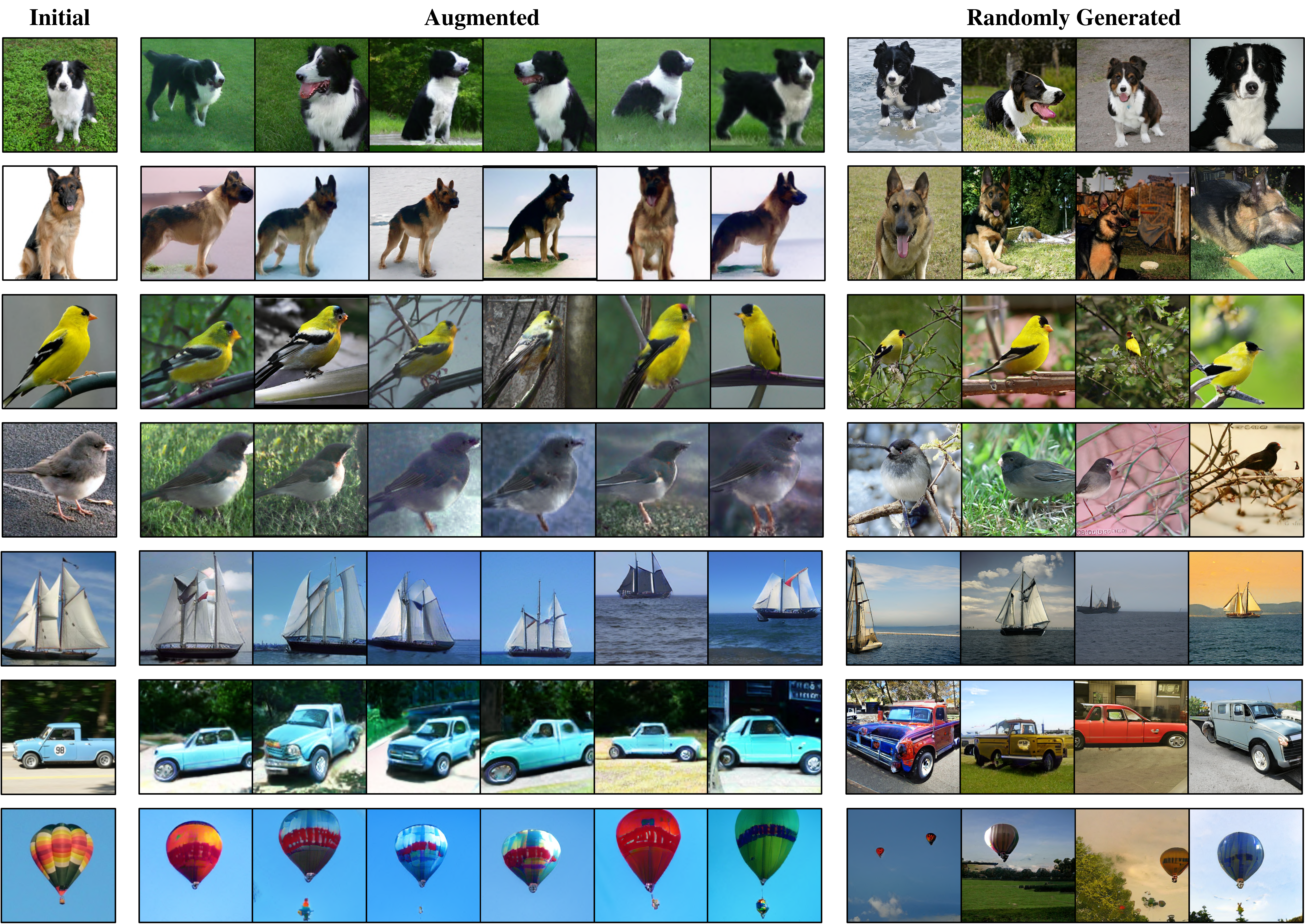}
    \vskip -0.1in
    \caption{Visualization of the semantically augmented images on ImageNet. ISDA is able to alter the semantics of images that are unrelated to the class identity, like backgrounds, actions of animals, visual angles, etc. We also present the randomly generated images of the same class.}
    \vskip -0.1in
    \label{fig:Visual_imagenet}
    \end{center}
    \vskip -0.1in
  \end{figure*}

\begin{figure*}
    \begin{center}
    \includegraphics[width=1.9\columnwidth]{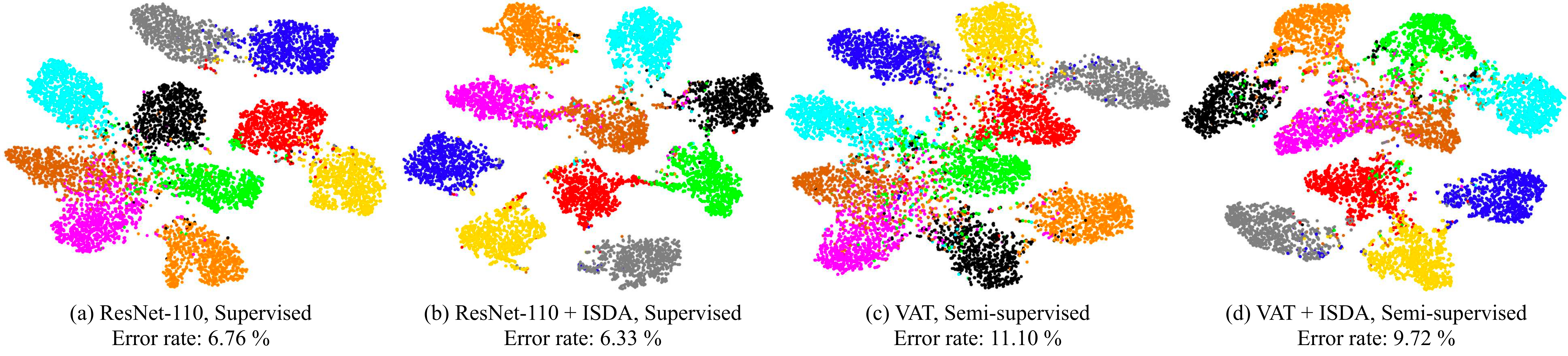}
    \vskip -0.1in
    \caption{Visualization of deep features on CIFAR-10 using the t-SNE algorithm \cite{maaten2008visualizing}. Each color denotes a class. (a), (b) present the results of supervised learning with ResNet-110, while (c), (d) present the results of semi-supervised learning with the VAT algorithm and 4000 labeled samples. The standard non-semantic data augmentation techniques are implemented.}
    \vskip -0.1in
    \label{tsne}
    \end{center}
    \vskip -0.1in
\end{figure*}



\subsection{Analytical Results}
\label{Analytical_results}
\subsubsection{Visualization}

\textbf{Visualization of augmented images. }
To demonstrate that our method is able to generate meaningful semantically augmented samples, we introduce an approach to map the augmented features back to the pixel space to explicitly show semantic changes of the images. Due to space limitations, we defer the detailed introduction of the mapping algorithm and present it in Appendix \ref{reverse_alg}. Figure \ref{Visual} and Figure \ref{fig:Visual_imagenet} show the visualization results. The first column represents the original images. The 'Augmented' columns present the images augmented by the proposed ISDA. It can be observed that ISDA is able to alter the semantics of images, e.g., backgrounds, visual angles, actions of dogs and color of skins, which is not possible for traditional data augmentation techniques. For a clear comparison, we also present the randomly generated images of the same class.


\textbf{Visualization of Deep Features.}
We visualize the learned deep features on CIFAR-10 with and without ISDA using the t-SNE algorithm \cite{maaten2008visualizing}. The results of both supervised learning with ResNet-110 and semi-supervised learning with VAT are presented. It can be observed that, with ISDA, the deep features of different classes form more tight and concentrated clusters. Intuitively, they are potentially more separable from each other. In contrast, features learned without ISDA are distributed in clusters that have many overlapped parts. 

\subsubsection{Tightness of the Upper Bound $\overline{\mathcal{L}}_{\infty}$}
As aforementioned, the proposed ISDA algorithm uses the upper bound of the expected loss as the the surrogate loss. Therefore, the upper bound $\overline{\mathcal{L}}_{\infty}$ is required to be tight enough to ensure that the expected loss is minimized. To check the tightness of $\overline{\mathcal{L}}_{\infty}$ in practice, we empirically compute ${\mathcal{L}}_{\infty}$ and $\overline{\mathcal{L}}_{\infty}$ over the training iterations of ResNet-110, shown in Figure \ref{loss_curve}. We can observe that $\overline{\mathcal{L}}_{\infty}$ gives a very tight upper bound on both CIFAR-10 and CIFAR-100.

\begin{figure}[t]
    \begin{center}
    \subfigure[CIFAR-10]{
        \label{fig:loss_curve_1}
    \includegraphics[width=0.49\columnwidth]{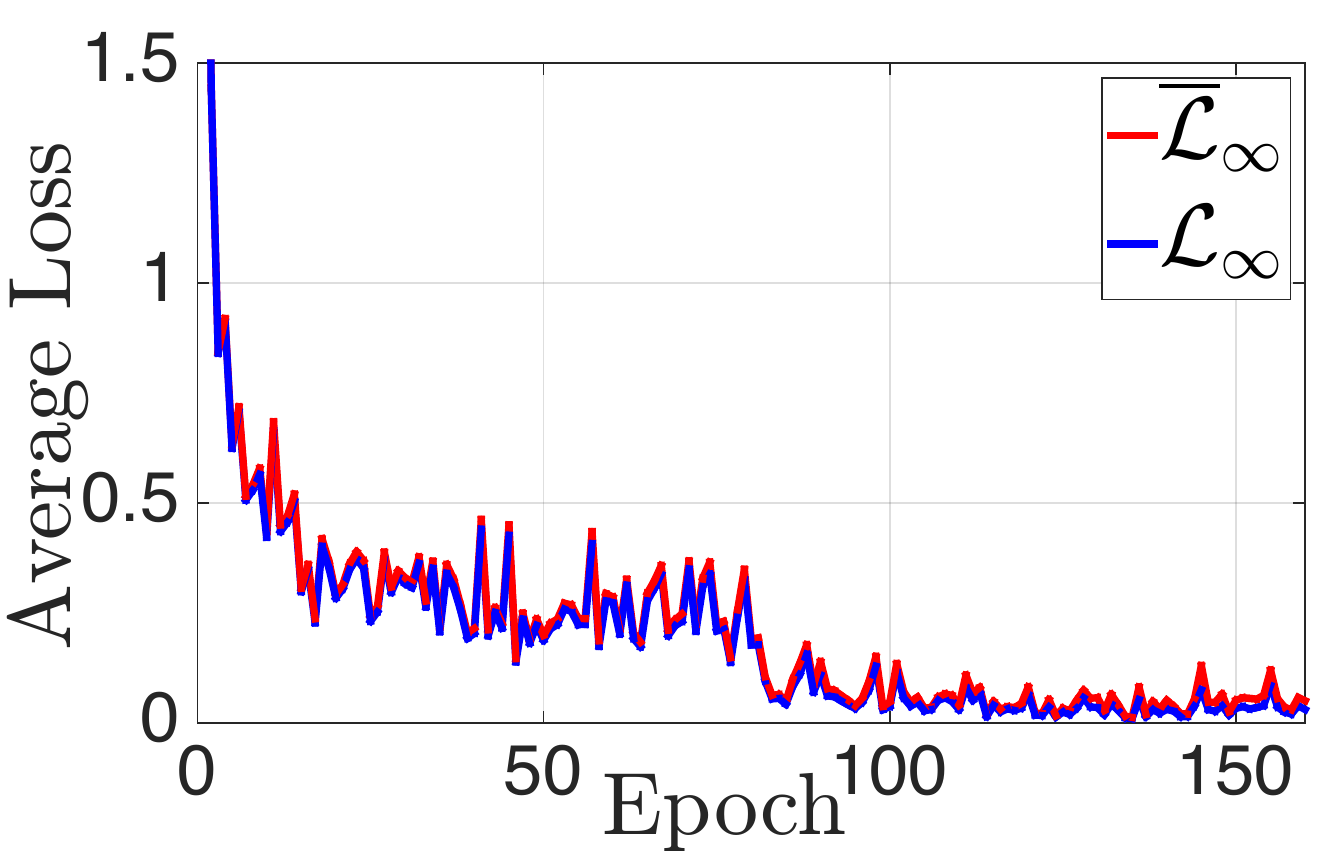}
    }
    \hspace{-0.15in}
    \subfigure[CIFAR-100]{
        \label{fig:loss_curve_2}
    \includegraphics[width=0.49\columnwidth]{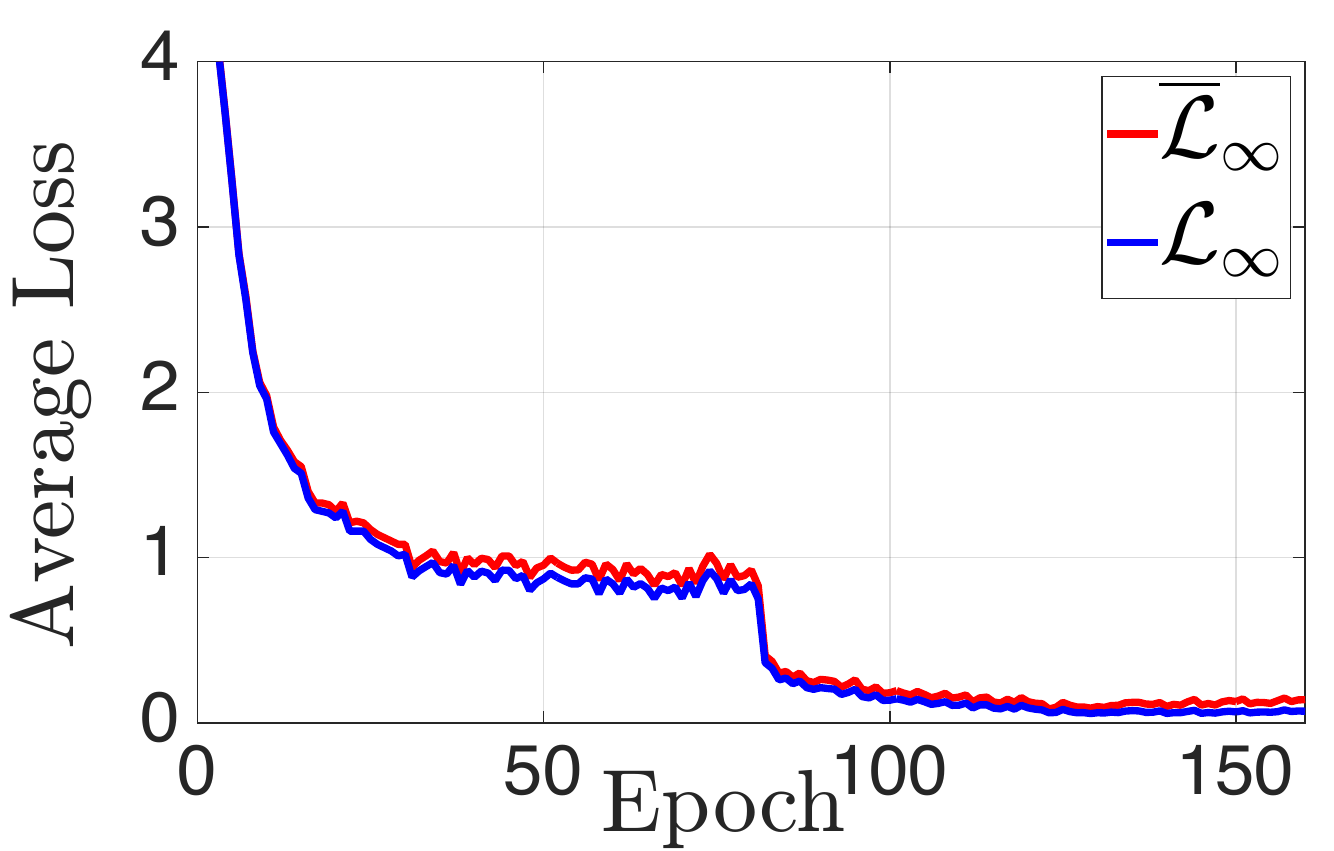}}
    \vskip -0.1in
    \caption{Values of $\mathcal{L}_{\infty}$ and $\overline{\mathcal{L}}_{\infty}$ over the training process. The value of ${\mathcal{L}}_{\infty}$ is estimated using Monte-Carlo sampling with a sample size of 1,000.}
    \label{loss_curve}
    \end{center}
    \vskip -0.2in
\end{figure}

\begin{figure}[t]
    \hspace{-0.4in}
    \begin{center}
    \subfigure[w/ Cutout]{
        \label{fig:samlp_M_cutout}
    \includegraphics[width=0.5\columnwidth]{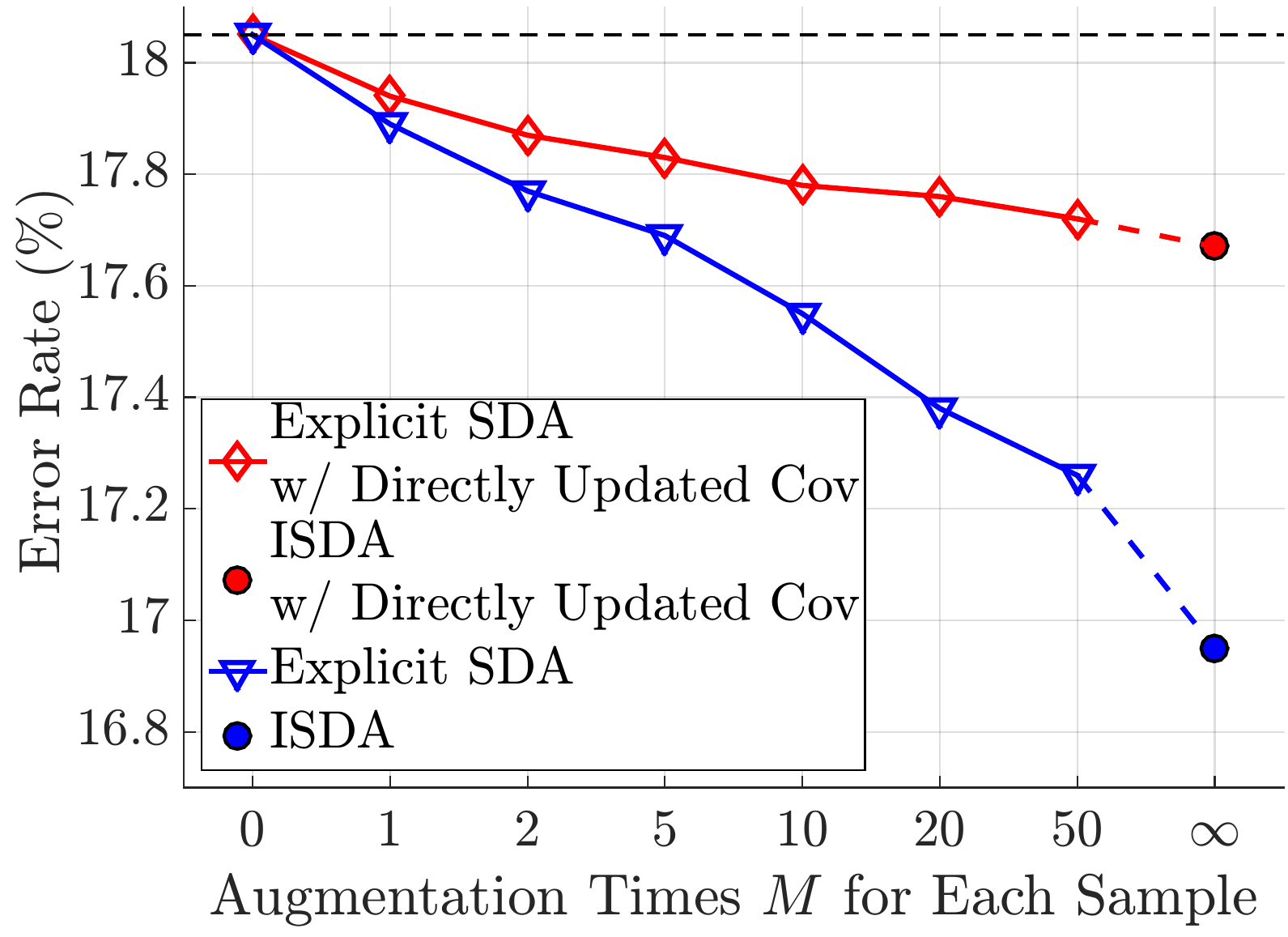}
    }
    \hspace{-0.2in}
    \subfigure[w/ AutoAugment]{
        \label{fig:samlp_M_AA}
    \includegraphics[width=0.5\columnwidth]{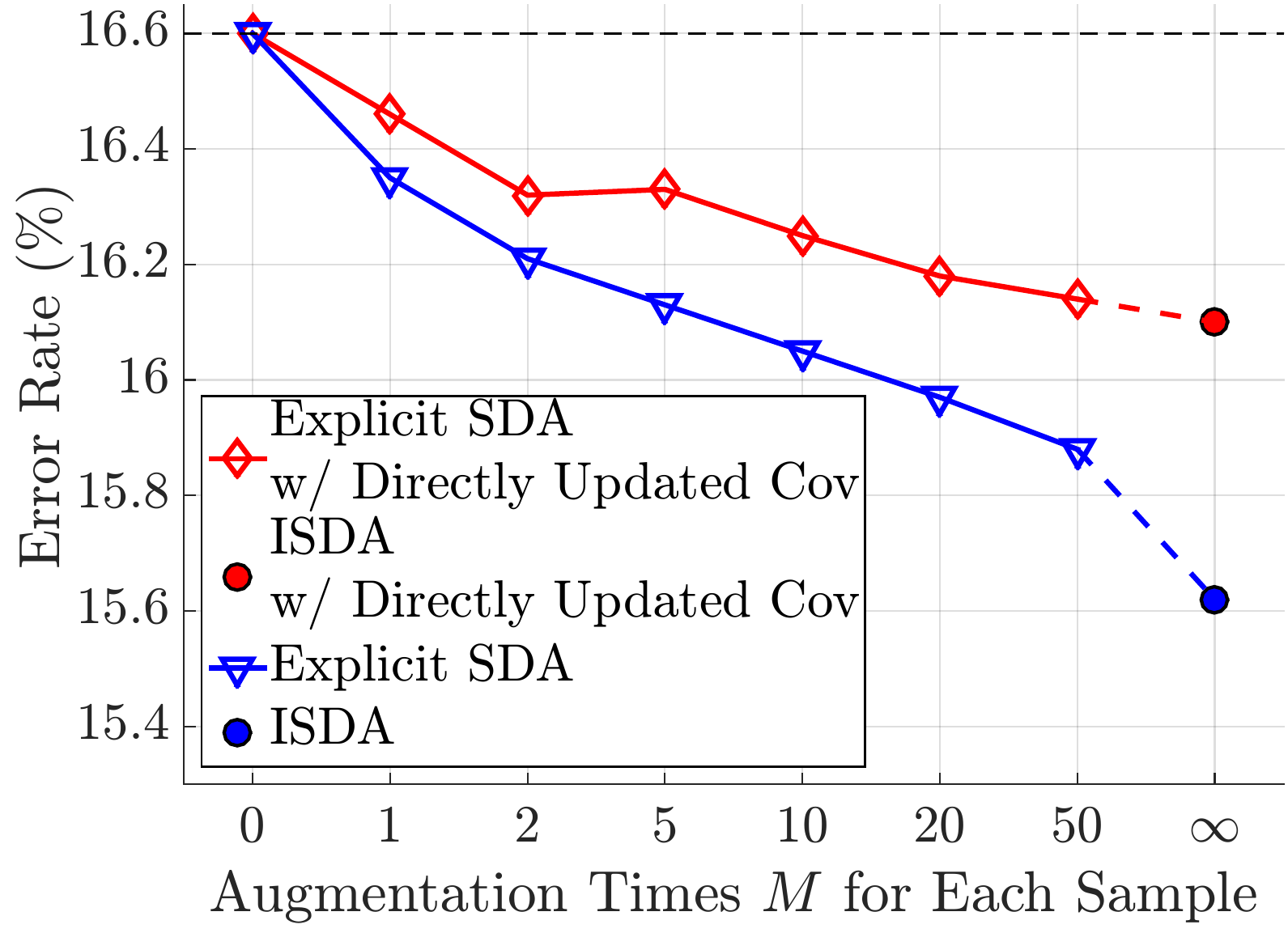}}
    \vskip -0.1in
    \caption{Comparisons of explicit semantic data augmentation (explicit SDA) and ISDA. For the former, we vary the value of the sample times $M$, and train the networks by minimizing Eq. (\ref{eq2}). As a baseline, we also consider directly updating the covariance matrices (Cov) $\Sigma_{1}$, $\Sigma_{2}$, $...$, $\Sigma_{C}$ with gradient decent. The results are presents in red lines. We report the test errors of Wide-ResNet-28-10 on CIFAR-100 with the Cutout and AutoAugment augmentation. $M=0$ refers to the baseline results, while $M=\infty$ refers to ISDA.}
    \label{fig:samlp_M}
    \end{center}
    \vskip -0.2in
\end{figure}

\subsubsection{Comparisons of Explicit and Implicit Semantic Data Augmentation}
To study whether the proposed upper bound $\overline{\mathcal{L}}_{\infty}$ leads to better performance than the sample-based explicit semantic data augmentation (i.e., explicit SDA, minimizing Eq. (\ref{eq2}) with certain sample times $M$), we compare these two approaches in Figure \ref{fig:samlp_M}. We also consider another baseline, namely learning the covariance matrices $\Sigma_{1}$, $\Sigma_{2}$, $...$, $\Sigma_{C}$ directly by gradient decent. For explicit SDA, this is achieved by the re-parameterization trick \cite{Kingma2014AutoEncodingVB, yu2019robust}. To ensure learning symmetrical positive semi-definite covariance matrices, we let $\Sigma_{i} = \bm{D}_{i}\bm{D}_{i}^{\textnormal{T}}$, and update $\bm{D}_{i}$ instead, which has the same size as $\Sigma_{i}$. In addition, to avoid trivially obtaining all-zero covariance matrices, we add the feature uncertainty loss in \cite{yu2019robust} to the loss function for encouraging the augmentation distributions with large entropy. Its coefficient is tuned on the validation set.

From the results one can observe that explicit SDA with small $M$ manages to reduce test errors, but the effects are less significant. This might be attributed to the high dimensions of feature space. For example, given that Wide-ResNet-28-10 produces 640-dimensional features, small sample numbers (e.g., $M=1,2,5$) may result in poor estimates of the expected loss. Accordingly, when $M$ grows larger, the performance of explicit SDA approaches ISDA, indicating that ISDA models the case of $M \to \infty$. On the other hand, we note that the dynamically estimated intra-class covariance matrices outperform the directly learned ones consistently for both explicit and implicit augmentation. We tentatively attribute this to the rich class-conditional semantic information captured by the former.

\subsubsection{Sensitivity Test}
To study how the hyper-parameter $\lambda_0$ affects the performance of our method, sensitivity tests are conducted for both supervised learning and semi-supervised learning. The results are shown in Figure \ref{sensitivity}.
It can be observed that ISDA achieves superior performance robustly with $ 0.25\!\leq\!\lambda_0\!\leq\!1$, and the error rates start to increase with $\lambda_0\!>\!1$. However, ISDA is still effective even when $\lambda_0$ grows to $5$, while it performs slightly worse than baselines when $\lambda_0$ reaches $10$.
Empirically, we recommend $\lambda_0\!=\!0.5$ for naive implementation or a start point of hyper-parameter searching.

\begin{figure}[t]
    \begin{center}
    \subfigure[Supervised learning]{
        \label{fig:sense_1}
    \includegraphics[width=0.49\columnwidth]{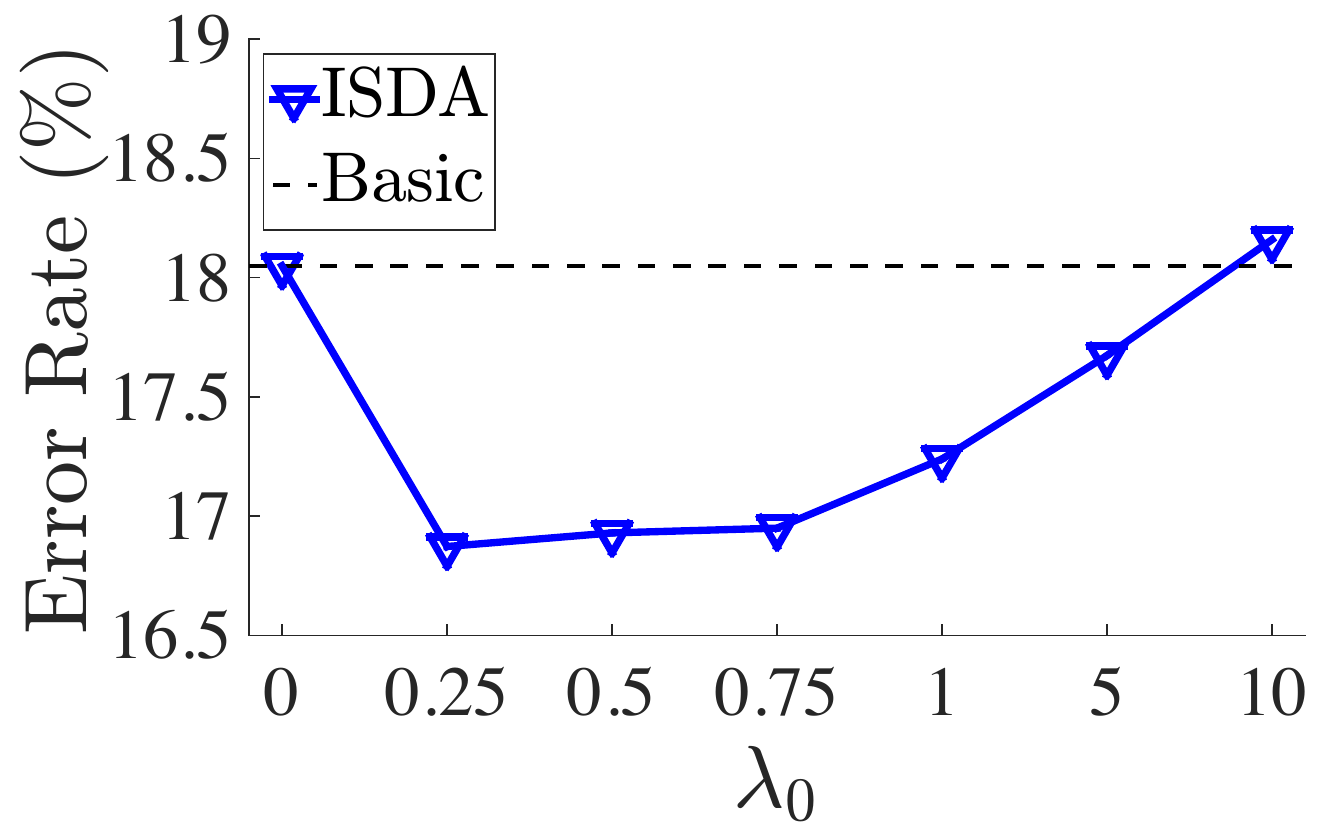}
    }
    \hspace{-0.15in}
    \subfigure[Semi-supervised learning]{
        \label{fig:sense_2}
    \includegraphics[width=0.49\columnwidth]{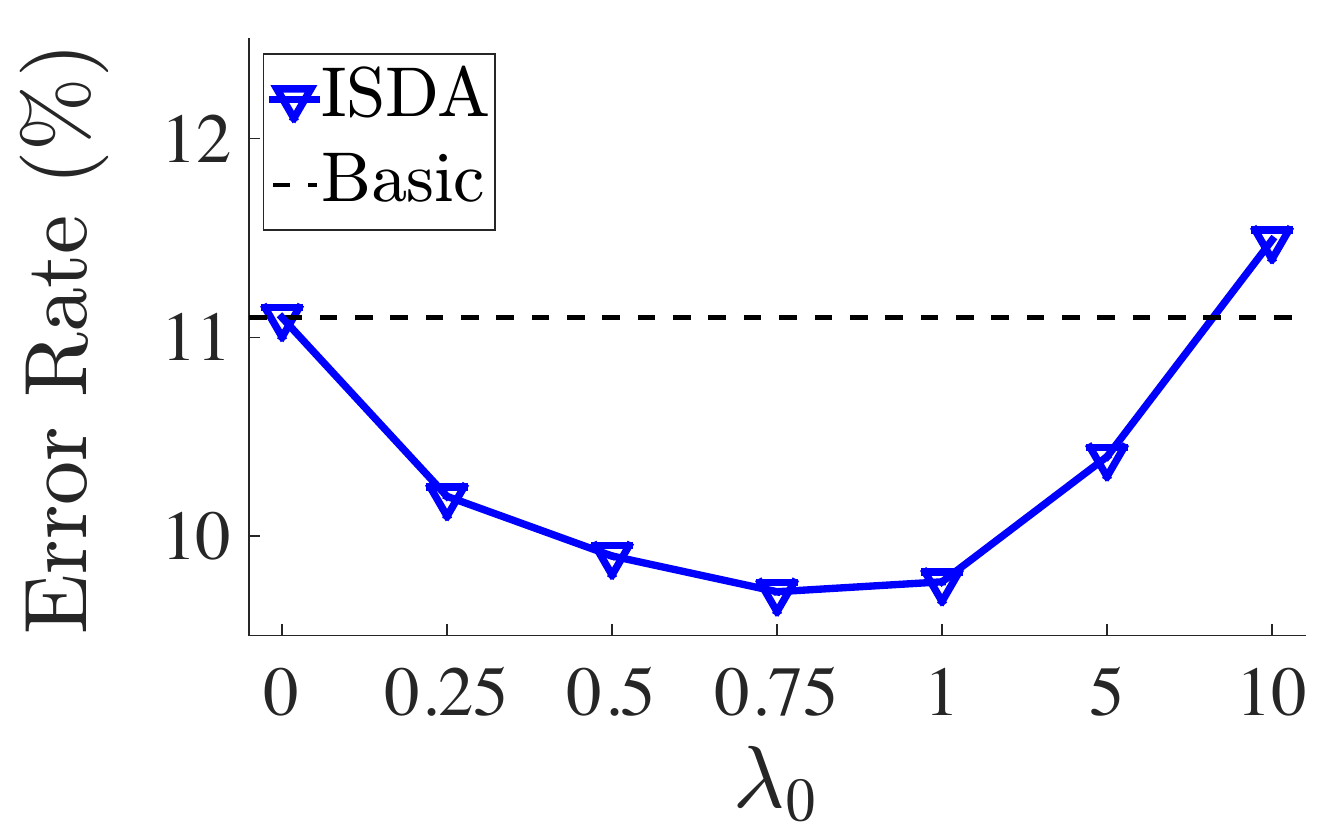}}
    \vskip -0.1in
    \caption{Sensitivity analysis of ISDA. For supervised learning, we report the test errors of Wide-ResNet-28-10 on CIFAR-100 with different values of $\lambda_0$. The Cutout augmentation is adopted. For semi-supervised learning, we present the results of VAT + ISDA on CIFAR-10 with 4,000 labels.}
    \label{sensitivity}
    \end{center}
    \vskip -0.2in
\end{figure}

\begin{table}[t]
    \centering
    \caption{The ablation study for ISDA.}
    \vskip -0.1in
    \label{Tab05}
    \setlength{\tabcolsep}{0.7mm}{
    \renewcommand\arraystretch{1.21} 
    \begin{tabular}{l|c|c|c}
        \hline
        \multirow{2}{*}{Setting} & \multirow{2}{*}{CIFAR-10} & \multirow{2}{*}{CIFAR-100} & CIFAR-100\\
        &&& + Cutout\\
            \hline
            Basic & 3.82 $\pm$ 0.15\% & 18.58 $\pm$ 0.10\% & 18.05 $\pm$ 0.25\%\\
            \hline
            Identity matrix & 3.63 $\pm$ 0.12\% & 18.53 $\pm$ 0.02\% & 17.83 $\pm$ 0.36\% \\
            Diagonal matrix & 3.70 $\pm$ 0.15\% & 18.23 $\pm$ 0.02\% & 17.54 $\pm$ 0.20\% \\
            Single covariance matrix & 3.67 $\pm$ 0.07\% & 18.29 $\pm$ 0.13\% & 18.12 $\pm$ 0.20\% \\
            Constant $\lambda_0$ & 3.69 $\pm$ 0.08\% & 18.33 $\pm$ 0.16\% & 17.34 $\pm$ 0.15\% \\
            \hline
            ISDA & \textbf{3.58 $\pm$ 0.15\%} & \textbf{17.98 $\pm$ 0.15\%} & \textbf{16.95  $\pm$  0.11\%} \\        
        \hline
    \end{tabular}}
    \vskip -0.1in
\end{table}

\subsubsection{Ablation Study}
To get a better understanding of the effectiveness of different components in ISDA, we conduct a series of ablation studies. In specific, several variants are considered: 
(1) \textit{Identity matrix} means replacing the covariance matrix $\Sigma_j$ by the identity matrix. 
(2) \textit{Diagonal matrix} means using only the diagonal elements of the covariance matrix $\Sigma_j$.
(3) \textit{Single covariance matrix} means using a global covariance matrix computed from the features of all classes.
(4) \textit{Constant $\lambda_0$} means using a constant $\lambda_0$ instead of a function of the training iterations. 

Table \ref{Tab05} presents the ablation results. Adopting the identity matrix increases the test error by 0.05\%, 0.55\% and 0.88\% on CIFAR-10, CIFAR-100 and CIFAR-100+Cutout, respectively. Using a single covariance matrix greatly degrades the generalization performance as well. The reason is likely to be that both of them fail to find proper directions in the deep feature space to perform meaningful semantic transformations. Adopting a diagonal matrix also hurts the performance as it does not consider correlations of features.

%% file: conclusion.tex
\section{Conclusion}
In this paper, we proposed an efficient implicit semantic data augmentation algorithm (ISDA) to complement existing data augmentation techniques. Different from existing approaches leveraging generative models to augment the training set with semantically transformed samples, our approach is considerably more efficient and easier to implement. In fact, we showed that ISDA can be formulated as a novel robust loss function, which is compatible with any deep network using the softmax cross-entropy loss. Additionally, ISDA can also be implemented efficiently in semi-supervised learning via the semantic consistency training technique. Extensive experimental results on several competitive vision benchmarks demonstrate the effectiveness and efficiency of the proposed algorithm.

%% file: bare_adv.bbl
\begin{thebibliography}{10}
\providecommand{\url}[1]{#1}
\csname url@samestyle\endcsname
\providecommand{\newblock}{\relax}
\providecommand{\bibinfo}[2]{#2}
\providecommand{\BIBentrySTDinterwordspacing}{\spaceskip=0pt\relax}
\providecommand{\BIBentryALTinterwordstretchfactor}{4}
\providecommand{\BIBentryALTinterwordspacing}{\spaceskip=\fontdimen2\font plus
\BIBentryALTinterwordstretchfactor\fontdimen3\font minus
  \fontdimen4\font\relax}
\providecommand{\BIBforeignlanguage}[2]{{%
\expandafter\ifx\csname l@#1\endcsname\relax
\typeout{** WARNING: IEEEtran.bst: No hyphenation pattern has been}%
\typeout{** loaded for the language `#1'. Using the pattern for}%
\typeout{** the default language instead.}%
\else
\language=\csname l@#1\endcsname
\fi
#2}}
\providecommand{\BIBdecl}{\relax}
\BIBdecl

\bibitem{krizhevsky2009learning}
A.~Krizhevsky and G.~Hinton, ``Learning multiple layers of features from tiny
  images,'' Citeseer, Tech. Rep., 2009.

\bibitem{krizhevsky2012imagenet}
A.~Krizhevsky, I.~Sutskever, and G.~E. Hinton, ``Imagenet classification with
  deep convolutional neural networks,'' in \emph{NeurIPS}, 2012, pp.
  1097--1105.

\bibitem{2014arXiv1409.1556S}
K.~{Simonyan} and A.~{Zisserman}, ``Very deep convolutional networks for
  large-scale image recognition,'' in \emph{ICLR}, 2015.

\bibitem{He_2016_CVPR}
K.~He, X.~Zhang, S.~Ren, and J.~Sun, ``Deep residual learning for image
  recognition,'' in \emph{CVPR}, 2016, pp. 770--778.

\bibitem{2016arXiv160806993H}
G.~{Huang}, Z.~{Liu}, G.~{Pleiss}, L.~{Van Der Maaten}, and K.~{Weinberger},
  ``Convolutional networks with dense connectivity,'' \emph{IEEE Transactions
  on Pattern Analysis and Machine Intelligence}.

\bibitem{NIPS2017_6916}
A.~J. Ratner, H.~Ehrenberg, Z.~Hussain, J.~Dunnmon, and C.~R\'{e}, ``Learning
  to compose domain-specific transformations for data augmentation,'' in
  \emph{NeurIPS}, 2017, pp. 3236--3246.

\bibitem{bowles2018gan}
C.~Bowles, L.~J. Chen, R.~Guerrero, P.~Bentley, R.~N. Gunn, A.~Hammers, D.~A.
  Dickie, M.~del C.~Vald{\'e}s~Hern{\'a}ndez, J.~M. Wardlaw, and D.~Rueckert,
  ``Gan augmentation: Augmenting training data using generative adversarial
  networks,'' \emph{CoRR}, vol. abs/1810.10863, 2018.

\bibitem{antoniou2017data}
A.~Antoniou, A.~J. Storkey, and H.~A. Edwards, ``Data augmentation generative
  adversarial networks,'' \emph{CoRR}, vol. abs/1711.04340, 2018.

\bibitem{Upchurch2017DeepFI}
P.~Upchurch, J.~R. Gardner, G.~Pleiss, R.~Pless, N.~Snavely, K.~Bala, and K.~Q.
  Weinberger, ``Deep feature interpolation for image content changes,'' in
  \emph{CVPR}, 2017, pp. 6090--6099.

\bibitem{bengio2013better}
Y.~Bengio, G.~Mesnil, Y.~Dauphin, and S.~Rifai, ``Better mixing via deep
  representations,'' in \emph{ICML}, 2013, pp. 552--560.

\bibitem{rasmus2015semi}
A.~Rasmus, M.~Berglund, M.~Honkala, H.~Valpola, and T.~Raiko, ``Semi-supervised
  learning with ladder networks,'' in \emph{NeurIPS}, 2015, pp. 3546--3554.

\bibitem{kingma2014semi}
D.~P. Kingma, S.~Mohamed, D.~J. Rezende, and M.~Welling, ``Semi-supervised
  learning with deep generative models,'' in \emph{NeurIPS}, 2014, pp.
  3581--3589.

\bibitem{miyato2018virtual}
T.~Miyato, S.-i. Maeda, M.~Koyama, and S.~Ishii, ``Virtual adversarial
  training: a regularization method for supervised and semi-supervised
  learning,'' \emph{IEEE transactions on pattern analysis and machine
  intelligence}, vol.~41, no.~8, pp. 1979--1993, 2018.

\bibitem{tarvainen2017mean}
A.~Tarvainen and H.~Valpola, ``Mean teachers are better role models:
  Weight-averaged consistency targets improve semi-supervised deep learning
  results,'' in \emph{NeurIPS}, 2017, pp. 1195--1204.

\bibitem{laine2016temporal}
S.~Laine and T.~Aila, ``Temporal ensembling for semi-supervised learning,''
  \emph{arXiv preprint arXiv:1610.02242}, 2016.

\bibitem{wang2019implicit}
Y.~{Wang}, X.~{Pan}, S.~{Song}, H.~{Zhang}, C.~{Wu}, and G.~{Huang},
  ``{Implicit Semantic Data Augmentation for Deep Networks},'' in
  \emph{NeurIPS}, 2019.

\bibitem{srivastava2015training}
R.~K. Srivastava, K.~Greff, and J.~Schmidhuber, ``Training very deep
  networks,'' in \emph{NeurIPS}, 2015, pp. 2377--2385.

\bibitem{devries2017improved}
T.~DeVries and G.~W. Taylor, ``Improved regularization of convolutional neural
  networks with cutout,'' \emph{arXiv preprint arXiv:1708.04552}, 2017.

\bibitem{zhong2017random}
Z.~Zhong, L.~Zheng, G.~Kang, S.~Li, and Y.~Yang, ``Random erasing data
  augmentation,'' \emph{arXiv preprint arXiv:1708.04896}, 2017.

\bibitem{2018arXiv180509501C}
E.~D. Cubuk, B.~Zoph, D.~Man{\'e}, V.~Vasudevan, and Q.~V. Le, ``Autoaugment:
  Learning augmentation policies from data,'' \emph{CoRR}, vol. abs/1805.09501,
  2018.

\bibitem{maaten2013learning}
L.~Maaten, M.~Chen, S.~Tyree, and K.~Weinberger, ``Learning with marginalized
  corrupted features,'' in \emph{ICML}, 2013, pp. 410--418.

\bibitem{yin2019feature}
X.~Yin, X.~Yu, K.~Sohn, X.~Liu, and M.~Chandraker, ``Feature transfer learning
  for face recognition with under-represented data,'' in \emph{CVPR}, 2019, pp.
  5704--5713.

\bibitem{jaderberg2016reading}
M.~Jaderberg, K.~Simonyan, A.~Vedaldi, and A.~Zisserman, ``Reading text in the
  wild with convolutional neural networks,'' \emph{International Journal of
  Computer Vision}, vol. 116, no.~1, pp. 1--20, 2016.

\bibitem{bousmalis2017unsupervised}
K.~Bousmalis, N.~Silberman, D.~Dohan, D.~Erhan, and D.~Krishnan, ``Unsupervised
  pixel-level domain adaptation with generative adversarial networks,'' in
  \emph{CVPR}, 2017, pp. 3722--3731.

\bibitem{Zhang2018GeneralizedCE}
Z.~Zhang and M.~R. Sabuncu, ``Generalized cross entropy loss for training deep
  neural networks with noisy labels,'' in \emph{NeurIPS}, 2018.

\bibitem{Lin2017FocalLF}
T.-Y. Lin, P.~Goyal, R.~B. Girshick, K.~He, and P.~Doll{\'a}r, ``Focal loss for
  dense object detection,'' in \emph{ICCV}, 2017, pp. 2999--3007.

\bibitem{liu2016large}
W.~Liu, Y.~Wen, Z.~Yu, and M.~Yang, ``Large-margin softmax loss for
  convolutional neural networks.'' in \emph{ICML}, 2016.

\bibitem{Liang2017SoftMarginSF}
X.~Liang, X.~Wang, Z.~Lei, S.~Liao, and S.~Z. Li, ``Soft-margin softmax for
  deep classification,'' in \emph{ICONIP}, 2017.

\bibitem{Wang2018EnsembleSS}
X.~Wang, S.~Zhang, Z.~Lei, S.~Liu, X.~Guo, and S.~Z. Li, ``Ensemble soft-margin
  softmax loss for image classification,'' in \emph{IJCAI}, 2018.

\bibitem{Sun2014DeepLF}
Y.~Sun, X.~Wang, and X.~Tang, ``Deep learning face representation by joint
  identification-verification,'' in \emph{NeurIPS}, 2014.

\bibitem{wen2016discriminative}
Y.~Wen, K.~Zhang, Z.~Li, and Y.~Qiao, ``A discriminative feature learning
  approach for deep face recognition,'' in \emph{ECCV}, 2016, pp. 499--515.

\bibitem{Bengio2009DeepArchitechture}
Y.~Bengio \emph{et~al.}, ``Learning deep architectures for ai,''
  \emph{Foundations and trends{\textregistered} in Machine Learning}, vol.~2,
  no.~1, pp. 1--127, 2009.

\bibitem{Choi2018StarGANUG}
Y.~Choi, M.-J. Choi, M.~Kim, J.-W. Ha, S.~Kim, and J.~Choo, ``Stargan: Unified
  generative adversarial networks for multi-domain image-to-image
  translation,'' in \emph{CVPR}, 2018, pp. 8789--8797.

\bibitem{zhu2017unpaired}
J.-Y. Zhu, T.~Park, P.~Isola, and A.~A. Efros, ``Unpaired image-to-image
  translation using cycle-consistent adversarial networks,'' in \emph{ICCV},
  2017, pp. 2223--2232.

\bibitem{He2018AttGANFA}
Z.~He, W.~Zuo, M.~Kan, S.~Shan, and X.~Chen, ``Attgan: Facial attribute editing
  by only changing what you want.'' \emph{CoRR}, vol. abs/1711.10678, 2017.

\bibitem{kendall2017uncertainties}
A.~Kendall and Y.~Gal, ``What uncertainties do we need in bayesian deep
  learning for computer vision?'' in \emph{NeurIPS}, 2017, pp. 5574--5584.

\bibitem{gal2015bayesian}
Y.~Gal and Z.~Ghahramani, ``Bayesian convolutional neural networks with
  bernoulli approximate variational inference,'' in \emph{ICLR}, 2015.

\bibitem{gal2016dropout}
------, ``Dropout as a bayesian approximation: Representing model uncertainty
  in deep learning,'' in \emph{ICML}, 2016, pp. 1050--1059.

\bibitem{shi2019probabilistic}
Y.~Shi and A.~K. Jain, ``Probabilistic face embeddings,'' in \emph{ICCV}, 2019,
  pp. 6902--6911.

\bibitem{yu2019robust}
T.~Yu, D.~Li, Y.~Yang, T.~M. Hospedales, and T.~Xiang, ``Robust person
  re-identification by modelling feature uncertainty,'' in \emph{ICCV}, 2019,
  pp. 552--561.

\bibitem{kendall2018multi}
A.~Kendall, Y.~Gal, and R.~Cipolla, ``Multi-task learning using uncertainty to
  weigh losses for scene geometry and semantics,'' in \emph{CVPR}, 2018, pp.
  7482--7491.

\bibitem{he2019bounding}
Y.~He, C.~Zhu, J.~Wang, M.~Savvides, and X.~Zhang, ``Bounding box regression
  with uncertainty for accurate object detection,'' in \emph{CVPR}, 2019, pp.
  2888--2897.

\bibitem{ren2015faster}
S.~Ren, K.~He, R.~Girshick, and J.~Sun, ``Faster r-cnn: Towards real-time
  object detection with region proposal networks,'' in \emph{NeurIPS}, 2015,
  pp. 91--99.

\bibitem{Li2016ConvolutionalNF}
M.~Li, W.~Zuo, and D.~Zhang, ``Convolutional network for attribute-driven and
  identity-preserving human face generation,'' \emph{CoRR}, vol.
  abs/1608.06434, 2016.

\bibitem{arjovsky2017wasserstein}
M.~Arjovsky, S.~Chintala, and L.~Bottou, ``Wasserstein gan,'' \emph{CoRR}, vol.
  abs/1701.07875, 2017.

\bibitem{luo2018smooth}
Y.~Luo, J.~Zhu, M.~Li, Y.~Ren, and B.~Zhang, ``Smooth neighbors on teacher
  graphs for semi-supervised learning,'' in \emph{CVPR}, 2018, pp. 8896--8905.

\bibitem{5206848}
J.~Deng, W.~Dong, R.~Socher, L.~Li, K.~Li, and L.~Fei-Fei, ``Imagenet: A
  large-scale hierarchical image database,'' in \emph{ICML}, 2009, pp.
  248--255.

\bibitem{goodfellow2013multi}
I.~J. Goodfellow, Y.~Bulatov, J.~Ibarz, S.~Arnoud, and V.~Shet, ``Multi-digit
  number recognition from street view imagery using deep convolutional neural
  networks,'' \emph{arXiv preprint arXiv:1312.6082}, 2013.

\bibitem{cordts2016cityscapes}
M.~Cordts, M.~Omran, S.~Ramos, T.~Rehfeld, M.~Enzweiler, R.~Benenson,
  U.~Franke, S.~Roth, and B.~Schiele, ``The cityscapes dataset for semantic
  urban scene understanding,'' in \emph{CVPR}, 2016, pp. 3213--3223.

\bibitem{lin2014microsoft}
T.-Y. Lin, M.~Maire, S.~Belongie, J.~Hays, P.~Perona, D.~Ramanan,
  P.~Doll{\'a}r, and C.~L. Zitnick, ``Microsoft coco: Common objects in
  context,'' in \emph{ECCV}.\hskip 1em plus 0.5em minus 0.4em\relax Springer,
  2014, pp. 740--755.

\bibitem{wang2020collaborative}
Y.~Wang, R.~Huang, G.~Huang, S.~Song, and C.~Wu, ``Collaborative learning with
  corrupted labels,'' \emph{Neural Networks}, vol. 125, pp. 205--213, 2020.

\bibitem{wang2021revisiting}
Y.~Wang, Z.~Ni, S.~Song, L.~Yang, and G.~Huang, ``Revisiting locally supervised
  learning: an alternative to end-to-end training,'' in \emph{ICLR}, 2021.

\bibitem{verma2019interpolation}
V.~Verma, A.~Lamb, J.~Kannala, Y.~Bengio, and D.~Lopez-Paz, ``Interpolation
  consistency training for semi-supervised learning,'' in \emph{IJCAI}, 2019.

\bibitem{yang2020resolution}
L.~Yang, Y.~Han, X.~Chen, S.~Song, J.~Dai, and G.~Huang, ``Resolution adaptive
  networks for efficient inference,'' in \emph{CVPR}, 2020, pp. 2369--2378.

\bibitem{wang2020glance}
Y.~Wang, K.~Lv, R.~Huang, S.~Song, L.~Yang, and G.~Huang, ``Glance and focus: a
  dynamic approach to reducing spatial redundancy in image classification,'' in
  \emph{NeurIPS}, 2020.

\bibitem{yun2019cutmix}
S.~Yun, D.~Han, S.~J. Oh, S.~Chun, J.~Choe, and Y.~Yoo, ``Cutmix:
  Regularization strategy to train strong classifiers with localizable
  features,'' in \emph{CVPR}, 2019, pp. 6023--6032.

\bibitem{xie2017aggregated}
S.~Xie, R.~Girshick, P.~Doll{\'a}r, Z.~Tu, and K.~He, ``Aggregated residual
  transformations for deep neural networks,'' in \emph{CVPR}, 2017, pp.
  1492--1500.

\bibitem{hu2018squeeze}
J.~Hu, L.~Shen, and G.~Sun, ``Squeeze-and-excitation networks,'' in
  \emph{CVPR}, 2018, pp. 7132--7141.

\bibitem{Zagoruyko2016WideRN}
S.~Zagoruyko and N.~Komodakis, ``Wide residual networks,'' in \emph{BMVC},
  2017.

\bibitem{gastaldi2017shake}
X.~Gastaldi, ``Shake-shake regularization,'' \emph{arXiv preprint
  arXiv:1705.07485}, 2017.

\bibitem{cubuk2020randaugment}
E.~D. Cubuk, B.~Zoph, J.~Shlens, and Q.~V. Le, ``Randaugment: Practical
  automated data augmentation with a reduced search space,'' in \emph{NeurIPS},
  2020, pp. 702--703.

\bibitem{cubuk2018autoaugment}
E.~D. Cubuk, B.~Zoph, D.~Mane, V.~Vasudevan, and Q.~V. Le, ``Autoaugment:
  Learning augmentation policies from data,'' in \emph{CVPR}, 2019.

\bibitem{Xie2016DisturbLabelRC}
L.~Xie, J.~Wang, Z.~Wei, M.~Wang, and Q.~Tian, ``Disturblabel: Regularizing cnn
  on the loss layer,'' in \emph{CVPR}, 2016, pp. 4753--4762.

\bibitem{muller2019does}
R.~M{\"u}ller, S.~Kornblith, and G.~E. Hinton, ``When does label smoothing
  help?'' in \emph{NeurIPS}, 2019, pp. 4694--4703.

\bibitem{mirza2014conditional}
M.~Mirza and S.~Osindero, ``Conditional generative adversarial nets,''
  \emph{CoRR}, vol. abs/1411.1784, 2014.

\bibitem{odena2017conditional}
A.~Odena, C.~Olah, and J.~Shlens, ``Conditional image synthesis with auxiliary
  classifier gans,'' in \emph{ICML}, 2017, pp. 2642--2651.

\bibitem{chen2016infogan}
X.~Chen, Y.~Duan, R.~Houthooft, J.~Schulman, I.~Sutskever, and P.~Abbeel,
  ``Infogan: Interpretable representation learning by information maximizing
  generative adversarial nets,'' in \emph{NeurIPS}, 2016, pp. 2172--2180.

\bibitem{Srivastava2014DropoutAS}
N.~Srivastava, G.~E. Hinton, A.~Krizhevsky, I.~Sutskever, and R.~R.
  Salakhutdinov, ``Dropout: a simple way to prevent neural networks from
  overfitting,'' \emph{Journal of Machine Learning Research}, vol.~15, pp.
  1929--1958, 2014.

\bibitem{xie2019unsupervised}
Q.~Xie, Z.~Dai, E.~Hovy, M.-T. Luong, and Q.~V. Le, ``Unsupervised data
  augmentation,'' \emph{arXiv preprint arXiv:1904.12848}, 2019.

\bibitem{wang2020meta}
Y.~Wang, J.~Guo, S.~Song, and G.~Huang, ``Meta-semi: A meta-learning approach
  for semi-supervised learning,'' \emph{arXiv preprint arXiv:2007.02394}, 2020.

\bibitem{zhao2017pyramid}
H.~Zhao, J.~Shi, X.~Qi, X.~Wang, and J.~Jia, ``Pyramid scene parsing network,''
  in \emph{CVPR}, 2017, pp. 2881--2890.

\bibitem{chen2017rethinking}
L.-C. Chen, G.~Papandreou, F.~Schroff, and H.~Adam, ``Rethinking atrous
  convolution for semantic image segmentation,'' \emph{arXiv preprint
  arXiv:1706.05587}, 2017.

\bibitem{huang2019ccnet}
Z.~Huang, X.~Wang, L.~Huang, C.~Huang, Y.~Wei, and W.~Liu, ``Ccnet: Criss-cross
  attention for semantic segmentation,'' in \emph{ICCV}, 2019, pp. 603--612.

\bibitem{liu2020structured}
Y.~Liu, C.~Shu, J.~Wang, and C.~Shen, ``Structured knowledge distillation for
  dense prediction,'' \emph{IEEE Transactions on Pattern Analysis and Machine
  Intelligence}, 2020.

\bibitem{maaten2008visualizing}
L.~v.~d. Maaten and G.~Hinton, ``Visualizing data using t-sne,'' \emph{Journal
  of machine learning research}, vol.~9, pp. 2579--2605, 2008.

\bibitem{Kingma2014AutoEncodingVB}
D.~P. Kingma and M.~Welling, ``Auto-encoding variational bayes,'' \emph{CoRR},
  vol. abs/1312.6114, 2014.

\bibitem{loshchilov2016sgdr}
I.~Loshchilov and F.~Hutter, ``Sgdr: Stochastic gradient descent with warm
  restarts,'' \emph{arXiv preprint arXiv:1608.03983}, 2016.

\bibitem{he2017mask}
K.~He, G.~Gkioxari, P.~Doll{\'a}r, and R.~Girshick, ``Mask r-cnn,'' in
  \emph{ICCV}, 2017, pp. 2961--2969.

\bibitem{cai2018cascade}
Z.~Cai and N.~Vasconcelos, ``Cascade r-cnn: Delving into high quality object
  detection,'' in \emph{CVPR}, 2018, pp. 6154--6162.

\bibitem{mmdetection}
K.~Chen, J.~Wang, J.~Pang, Y.~Cao, Y.~Xiong, X.~Li, S.~Sun, W.~Feng, Z.~Liu,
  J.~Xu, Z.~Zhang, D.~Cheng, C.~Zhu, T.~Cheng, Q.~Zhao, B.~Li, X.~Lu, R.~Zhu,
  Y.~Wu, J.~Dai, J.~Wang, J.~Shi, W.~Ouyang, C.~C. Loy, and D.~Lin,
  ``{MMDetection}: Open mmlab detection toolbox and benchmark,'' \emph{arXiv
  preprint arXiv:1906.07155}, 2019.

\bibitem{lin2017feature}
T.-Y. Lin, P.~Doll{\'a}r, R.~Girshick, K.~He, B.~Hariharan, and S.~Belongie,
  ``Feature pyramid networks for object detection,'' in \emph{CVPR}, 2017, pp.
  2117--2125.

\bibitem{mahendran2015understanding}
A.~Mahendran and A.~Vedaldi, ``Understanding deep image representations by
  inverting them,'' in \emph{CVPR}, 2015, pp. 5188--5196.

\bibitem{brock2018large}
A.~Brock, J.~Donahue, and K.~Simonyan, ``Large scale gan training for high
  fidelity natural image synthesis,'' \emph{arXiv preprint arXiv:1809.11096},
  2018.

\end{thebibliography}
